\documentclass{article} % For LaTeX2e
\usepackage{arxiv_preprint,times}

\usepackage{amsmath,amsfonts,bm}

\def\eqref#1{equation~\ref{#1}}
\def\1{\bm{1}}

\DeclareMathAlphabet{\mathsfit}{\encodingdefault}{\sfdefault}{m}{sl}
\SetMathAlphabet{\mathsfit}{bold}{\encodingdefault}{\sfdefault}{bx}{n}

\usepackage[hypertexnames=false]{hyperref}
\usepackage{url}
\usepackage{graphicx}
\usepackage{booktabs}
\usepackage{array}
\usepackage{colortbl}
\usepackage{multirow}
\usepackage{wrapfig}
\usepackage{amssymb}
\usepackage{amsmath}
\usepackage{amsthm}
\usepackage{algorithm}
\usepackage{algpseudocode}
\usepackage{placeins}

\newtheorem{proposition}{Proposition}

\theoremstyle{definition}
\newtheorem{definition}{Definition}

\title{Gradient--Update Mismatch: Rethinking Conflict-Free Training of\\[-0.15em]
Physics-Informed Neural Networks}

\author{
\begin{minipage}{0.98\textwidth}
\centering
{\fontsize{9}{10}\selectfont
\normalfont
\textbf{Jing Xiao}\textsuperscript{1,2,3} \quad
\textbf{Xinhai Chen}\textsuperscript{1,2,3,*} \quad
\textbf{Qinglin Wang}\textsuperscript{1,2,3} \quad
\textbf{Menghan Jia}\textsuperscript{2,3} \\
\textbf{Zhiquan Lai}\textsuperscript{2,3} \quad
\textbf{Dongsheng Li}\textsuperscript{2,3} \quad
\textbf{Jie Liu}\textsuperscript{1,2,3} \quad
\textbf{Tiejun Li}\textsuperscript{3}}
\\[0.3em]
{\fontsize{8}{9}\selectfont\itshape
\renewcommand{\arraystretch}{0.95}%
\begin{tabular}{@{}c@{}}
\textsuperscript{1}Laboratory of Digitizing Software for Frontier Equipment,\\
National University of Defense Technology, Changsha 410073, China\\[0.10em]
\textsuperscript{2}National Key Laboratory of Parallel and Distributed Computing,\\
National University of Defense Technology, Changsha 410073, China\\[0.10em]
\textsuperscript{3}School of Computer Science and Technology,\\
National University of Defense Technology, Changsha 410073, China\\[0.10em]
\textsuperscript{*}Corresponding author: chenxinhai16@nudt.edu.cn
\end{tabular}}
\end{minipage}
}

\makeatletter
\def\@maketitle{%
  \vbox{\hsize\textwidth
    \centering
    {\fontsize{14}{16}\selectfont\bfseries\rmfamily\@title\par}
    \vskip 0.12in
    \@author\par
    \vskip 0.14in}}
\makeatother

\newcommand{\imprsingle}[1]{\bfseries\makebox[2.4em][c]{#1}}
\newcommand{\imphead}{\makebox[2.4em][c]{Imp.}}
\newcommand{\qtwoentry}[2]{\makebox[3.2em][r]{#1}\(\pm\)\makebox[3.2em][l]{#2}}
\preprintfinalcopy
\begin{document}

\maketitle

% Keep the author-visible branch and use a neutral arXiv header.
\fancyhf{}
\fancyhead[C]{\small arXiv Preprint}
\fancyfoot[C]{\thepage}
\renewcommand{\headrulewidth}{0.4pt}

\begin{abstract}
Training Physics-Informed Neural Networks (PINNs) requires jointly optimizing
physics residual and initial/boundary condition loss terms, which often induce
conflicting gradients. Gradient surgery methods mitigate this issue by constructing directions from
loss-specific gradients to reduce conflict before optimizer transformation.
However, even when the constructed direction is conflict-free, this property
may not be preserved after optimizer transformation. Let $a_t$ denote the
direction constructed by gradient surgery, $u_t$ the optimizer proposal, and
$\mathcal C_t$ the conflict-free cone induced by the loss-specific gradients.
We show that modern optimizers can transform $a_t$ through mechanisms such as
historical state, adaptive scaling,  preconditioning, or decoupled weight decay, so
$a_t\in\mathcal C_t$ does not generally imply $u_t\in\mathcal C_t$.
We refer to this optimizer-induced discrepancy in conflict-freeness between
$a_t$ and $u_t$ as \emph{Gradient--Update Mismatch} (GUM).
Accordingly, we propose \emph{Gradient--Update Alignment} (GUA), which projects
$u_t$ onto $\mathcal C_t$ to obtain the aligned update $p_t$ and applies $p_t$
to the parameters. When the optimizer maintains internal state, GUA further
adjusts this state toward targets reconstructed from the applied update.
We conduct extensive experiments and find that GUM is widespread across momentum, adaptive, and curvature-based optimizers, with conflict rates reaching up to 86.3\%.
Across all PINN settings, GUA achieves conflict-free applied updates and
consistently improves various gradient surgery methods, reducing the relative
$L_2$ error by up to 98.2\% in individual settings.
Data and code are available at
\url{https://github.com/JingXiao10/GUA}.
\end{abstract}

\section{Introduction}
Solving partial differential equations (PDEs) can be computationally expensive,
especially for high-dimensional problems or scenarios requiring repeated
solution queries~\cite{li2020fourier,lu2021learning,kovachki2023neural}.
Physics-Informed Neural Networks (PINNs)~\cite{raissi2019physics} provide a
neural surrogate by incorporating physical constraints through loss terms
associated with PDE residuals and initial and boundary
conditions~\cite{gonon2024overview}. However, the gradients induced by different loss terms can differ substantially in magnitude and direction, leading to gradient conflicts and making PINN training difficult~\cite{zhang2026physics}.

A common strategy for improving PINN training is loss balancing, which adjusts the weights of different loss terms to reduce imbalances in gradient magnitudes among physical constraints. These weights can be set manually or adapted during
training. However, loss balancing does not directly address conflicts in
gradient directions, and the choice of effective weights varies across
problems~\cite{wang2021understanding,liu2021dual,li2022dynamic,bischof2025multi}.
This motivates the use of gradient surgery, which directly modifies
loss-specific gradients to construct directions with reduced conflict. Such
methods were originally developed for multi-task learning (MTL), including
PCGrad~\cite{yu2020gradient}, CAGrad~\cite{liu2021conflict},
IMTL-G~\cite{liu2021towards}, A-MTL~\cite{senushkin2023independent},
UPGrad~\cite{quinton2024jacobian}, and ConFIG~\cite{liu2025config}. These methods have also been adopted in PINNs to mitigate conflicts among the gradients induced by different physical constraints~\cite{liu2025config}.

Despite their methodological differences, these gradient surgery methods follow the same optimization pipeline: they first construct a direction from loss-specific gradients and then pass it to the optimizer. This naturally raises a key question: when gradient surgery constructs a conflict-free direction, does it remain conflict-free after optimizer transformation?
\begin{wrapfigure}[18]{r}{0.57\textwidth}
\vspace{-1em}
\centering
\includegraphics[width=\linewidth]{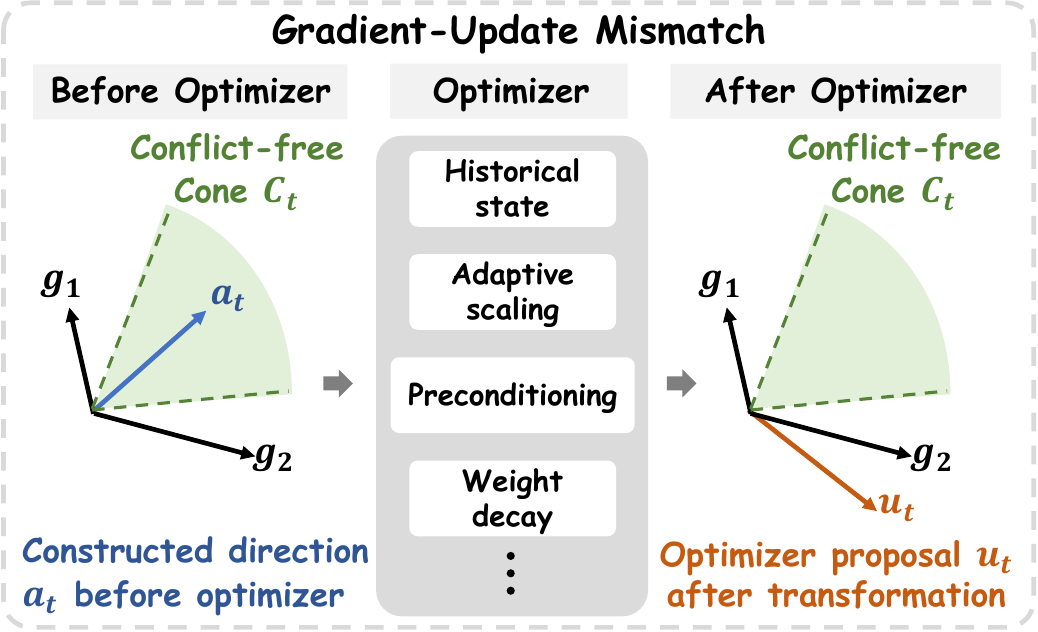}
\caption{Gradient--Update Mismatch. A conflict-free direction $a_t$ before
optimizer transformation may become a conflicting optimizer proposal $u_t$
after transformation.}
\label{fig:overview}
\end{wrapfigure}
Let $a_t$ denote the direction constructed by gradient surgery, $u_t$ the
optimizer proposal, and $\mathcal{C}_t$ the conflict-free cone induced by the
loss-specific gradients at step $t$. We show that a conflict-free direction
constructed by gradient surgery may not remain conflict-free after optimizer
transformation, as illustrated in Figure~\ref{fig:overview}. Formally,
$a_t \in \mathcal{C}_t$ does not generally imply $u_t \in \mathcal{C}_t$.
Modern optimizers can transform $a_t$ through mechanisms such as historical
state, adaptive scaling, preconditioning, or decoupled weight decay~\cite{wang2025gradient},
which can move the optimizer proposal outside the conflict-free cone. We refer
to this optimizer-induced discrepancy in conflict-freeness between $a_t$ and
$u_t$ as \emph{Gradient--Update Mismatch} (GUM).

This observation suggests that conflict handling should consider not only the
direction $a_t$ constructed before optimizer transformation, but also the
optimizer proposal $u_t$ before it is applied to the parameters.
We therefore propose \emph{Gradient--Update Alignment} (GUA), which aligns the
optimizer proposal $u_t$ with the current conflict-free cone $\mathcal C_t$.
Specifically, GUA projects $u_t$ onto $\mathcal C_t$ to obtain the aligned
update $p_t$, which is then applied to the parameters. When the optimizer
maintains internal state, GUA further adjusts this state toward targets
reconstructed from the applied update, reducing the influence of components removed by the
projection on subsequent proposals.

Our contributions are summarized as follows:

1. We identify and formalize a largely overlooked issue, which we term \emph{Gradient--Update Mismatch}: the conflict-free direction $a_t$ constructed by gradient surgery does not necessarily remain conflict-free after being transformed by the optimizer into the actual update proposal $u_t$.

2. We characterize how optimizer transformations can induce GUM. We derive exact
conflict-preservation conditions for fixed-state affine transformations and extend
the analysis to general nonlinear optimizer maps, covering mechanisms such as
historical state, adaptive scaling, preconditioning, and weight decay.

3. We propose \emph{Gradient--Update Alignment}, introducing an update-level perspective on conflict handling. GUA considers conflict-freeness not only for the direction constructed before optimizer transformation, but also for the optimizer proposal before it is applied to the parameters.

4. We conduct extensive experiments demonstrating that GUM is widespread across momentum, adaptive, and curvature-based optimizers, and that GUA achieves conflict-free applied updates across all evaluated PINN settings while consistently improving various gradient surgery methods.

\section{Related Work}
\subsection{PINN Training and Loss Balancing}
PINNs~\cite{raissi2019physics} are often difficult to train due to imbalanced gradient magnitudes, poor conditioning, and different convergence rates across loss
terms~\cite{krishnapriyan2021characterizing,wang2022and}. Loss-balancing methods
address part of this difficulty by manually or adaptively reweighting physical
constraints~\cite{wang2021understanding,liu2021dual,li2022dynamic,bischof2025multi}.
However, reweighting does not directly address conflicts in gradient
directions, motivating gradient surgery methods that operate directly on
loss-specific gradients.

\subsection{Gradient Surgery in Multi-Loss Optimization}
Multi-loss PINN training is closely related to multi-objective optimization and
gradient interference in multi-task learning, where different loss terms
can induce conflicting gradients
~\cite{sener2018multi,chen2018gradnorm,chen2020just}.
Gradient surgery methods address such conflicts by modifying loss-specific
gradients before optimizer transformation. PCGrad~\cite{yu2020gradient} reduces
pairwise conflicts through projection, CAGrad~\cite{liu2021conflict} balances
average descent with worst-case conflict, and IMTL-G~\cite{liu2021towards}
reduces task dominance by enforcing equal gradient projections. A-MTL
~\cite{senushkin2023independent} improves gradient conditioning and alignment,
while UPGrad~\cite{quinton2024jacobian} uses dual-cone projection to obtain a
common descent direction.
ConFIG~\cite{liu2025config,baldan2026physics} further constructs conflict-free directions
with non-negative inner products against loss-specific gradients for PINNs. Despite differing criteria, these methods all perform gradient surgery before optimizer transformation, overlooking the role of the optimizer in shaping the actual parameter update. This observation naturally motivates our work.

\subsection{Optimizer Transformations and Update-Level Alignment}

Modern optimizers generally transform the direction $a_t$ constructed by
gradient surgery before it is applied to the parameters. Such transformations may involve historical state, adaptive scaling, preconditioning, or decoupled weight decay.  Representative examples include momentum
SGD~\cite{sutskever2013importance}, RMSProp~\cite{tieleman2012rmsprop},
Adam~\cite{kingma2014adam}, AdamW~\cite{loshchilov2017decoupled}, and
curvature- or second-order-inspired methods such as
Sophia/SophiaG~\cite{liu2024sophia},
AdaHessian~\cite{yao2021adahessian}, and SOAP~\cite{vyas2025soap}.
Although such optimizer transformations can improve certain aspects of gradient
geometry during training~\cite{wang2025gradient}, they may not preserve the
conflict-free geometry established before optimizer transformation.
Our goal is not to design a new optimizer, but to extend conflict handling from the pre-optimizer direction to the optimizer proposal before it is applied to the parameters.
\section{Gradient--Update Mismatch and Gradient--Update Alignment}
\label{sec:gum_gua}

This section formalizes Gradient--Update Mismatch and introduces
Gradient--Update Alignment  for update-level conflict handling.

\subsection{Preliminaries: Multi-Loss Optimization}

\paragraph{Multi-loss objective.}
Consider an objective with \(m\) loss terms,
\(\mathcal{L}(\theta)=\sum_{i=1}^m\mathcal{L}_i(\theta)\),
where \(\theta\in\mathbb{R}^p\) denotes the model parameters and
\(\mathcal{L}_i\) denotes the \(i\)-th loss term, including any associated
weight. At optimization step $t$, the corresponding loss-specific gradient is
$
g_{i,t}=\nabla_\theta\mathcal L_i(\theta_t).
$

\paragraph{Conflict-free cone.}
For any candidate update direction \(d\), the linearized change of the \(i\)-th loss under
\(\theta_{t+1}=\theta_t-\eta d\) is
$
\mathcal L_i(\theta_t-\eta d)-\mathcal L_i(\theta_t)
=
-\eta\langle g_{i,t},d\rangle+O(\eta^2).
$
Therefore, \(d\) does not increase the \(i\)-th loss to first order if
$
\langle g_{i,t},d\rangle\ge0.
$
The directions that are compatible with all loss-specific gradients form the
conflict-free cone
\begin{equation}
\mathcal C_t
=
\{d\in\mathbb R^p\mid \langle g_{i,t},d\rangle\ge0,\ i=1,\ldots,m\}
=
\{d\in\mathbb R^p\mid G_t^\top d\ge0\},
\label{def:conflict}
\end{equation}
where \(G_t=[g_{1,t},\ldots,g_{m,t}]\). This cone characterizes the current
conflict-free geometry and will be used for update-level diagnosis and
projection.

\paragraph{Constructed direction and optimizer proposal.}
A gradient surgery method constructs a direction \(a_t\) from the
loss-specific gradients \(\{g_{i,t}\}_{i=1}^m\) to reduce gradient conflicts.
Let \(\mathcal O_t\) denote the optimizer transformation at step \(t\), and
let \(s_t\) denote its internal state, such as momentum or adaptive statistics.
The optimizer transforms \(a_t\) into a proposal \(u_t\) and a tentative next
state \(\bar s_{t+1}\):
\begin{equation}
(u_t,\bar s_{t+1})=\mathcal O_t(a_t,s_t).
\label{eq:optimizer_transformation}
\end{equation}
The optimizer proposal $u_t$, rather than the direction $a_t$  constructed by the gradient surgery method, determines the direction of the actual parameter update. Consequently, conflict-freeness of $a_t$ does not necessarily imply conflict-freeness of the actual update. We therefore introduce the notion of Gradient--Update Mismatch to characterize this optimizer-induced discrepancy in conflict-freeness.

\subsection{Gradient--Update Mismatch}
\label{sec:gum}

\begin{definition}[Gradient--Update Mismatch]
\label{def:gum}
At step $t$, a \emph{Gradient--Update Mismatch} (GUM) is an
optimizer-induced discrepancy in conflict-freeness between the constructed
direction $a_t$ and the optimizer proposal $u_t$. Specifically, GUM occurs when
$a_t \in \mathcal{C}_t$ while $u_t \notin \mathcal{C}_t$.
\end{definition}

To quantify conflicts throughout the optimization pipeline, we record
\begin{equation}
\begin{alignedat}{2}
R_g &= \frac{1}{T}\sum_{t=1}^{T}\mathbb{I}\!\left[\exists\, i<j:
\langle g_{i,t},g_{j,t}\rangle<0\right],
&\qquad R_a &= \frac{1}{T}\sum_{t=1}^{T}\mathbb{I}[a_t\notin\mathcal{C}_t],\\
R_u &= \frac{1}{T}\sum_{t=1}^{T}\mathbb{I}[u_t\notin\mathcal{C}_t],
&\qquad R_p &= \frac{1}{T}\sum_{t=1}^{T}\mathbb{I}[p_t\notin\mathcal{C}_t].
\end{alignedat}
\label{eq:conflict_rates}
\end{equation}

Here, $T$ is the total number of optimization steps, and $R_g$, $R_a$, $R_u$, and $R_p$ denote the conflict rates of the raw
loss-specific gradients, constructed direction, optimizer proposal, and aligned
update, respectively, yielding the diagnostic sequence
$R_g \rightarrow R_a \rightarrow R_u \rightarrow R_p$. For theoretical analysis, the above rates are defined using exact conflict conditions with a zero threshold. In experiments, we use a small normalized tolerance only in the conflict-rate diagnostics to account for numerical errors. Detailed specifications are given in Appendix~\ref{app:evaluation_metrics}.
Importantly, the relation between $R_a$ and $R_u$ depends on the guarantee provided by the gradient surgery method. If the method always produces conflict-free directions, then $R_a=0$, and every conflicting optimizer proposal counted by $R_u$ corresponds directly to a GUM event. Thus, under this conflict-free guarantee, any $R_u>0$ provides direct evidence of GUM.

\paragraph{Generalized GUM.}
Some gradient surgery methods do not guarantee conflict-free constructed
directions, so $R_a$ can be nonzero. In this case, the strict event definition
in Definition~\ref{def:gum} does not by itself summarize the aggregate change
in conflict frequency introduced by optimizer transformation. We therefore use
$R_u>R_a$ as a rate-level diagnostic of \emph{generalized GUM}, indicating that
optimizer transformation introduces a net increase in the frequency of
conflicting proposals. This is an aggregate diagnostic rather than an
event-level definition.

\subsection{Conflict Preservation under Optimizer Transformations}
\label{sec:conflict_preservation}

GUM arises when optimizer transformation fails to preserve the conflict-free
geometry of the constructed direction. This motivates us to characterize the
conditions under which such geometry is preserved. To isolate the geometric
mechanisms relevant to conflict preservation, we consider a fixed-state affine
representation at step $t$,
$
T_t(a)=P_ta+b_t.
$
Here, $P_t$ captures scaling and preconditioning effects, while $b_t$ represents an additive term that is independent of the input direction $a$, such as contributions arising from historical state or decoupled weight decay.

\begin{proposition}[Conflict preservation under fixed-state affine transformations]
\label{prop:conflict_preserving}
Let $G_t=[g_{1,t},\ldots,g_{m,t}]$ and
$\mathcal C_t=\{d\in\mathbb R^p\mid G_t^\top d\ge0\}$.
The affine transformation $T_t(a)=P_ta+b_t$ preserves conflict-freeness on
$\mathcal C_t$, i.e., $a\in\mathcal C_t\Rightarrow T_t(a)\in\mathcal C_t$,
if and only if
\begin{equation}
b_t\in\mathcal C_t,
\qquad
P_t^\top g_{i,t}\in
\operatorname{cone}(g_{1,t},\ldots,g_{m,t}),
\quad i=1,\ldots,m.
\label{eq:conflict_preserving_condition}
\end{equation}
\end{proposition}
The proof is provided in Appendix~\ref{app:proof_conflict_preserving}.
Within the fixed-state affine class,
Proposition~\ref{prop:conflict_preserving} exactly characterizes conflict
preservation by separating the effects of the additive term $b_t$ and the
linear transformation $P_t$. However, these preservation conditions need not
hold in general. Positive-scalar transformations with $P_t=\alpha I$,
$\alpha>0$, and $b_t=0$ preserve the cone, with plain SGD corresponding to
the special case $P_t=I$. In contrast, positive coordinate-wise scaling need
not preserve the cone. For example, let $g_{1,t}=(1,0)^\top$, $g_{2,t}=(1,1)^\top$, and
$a_t=(1,-0.5)^\top$. Under $P_t=\operatorname{diag}(0.1,10)$, the resulting
proposal is $u_t=(0.1,-5)^\top$, for which
$\langle g_{2,t},u_t\rangle=-4.9<0$. Therefore,
$u_t\notin\mathcal C_t$, even though $a_t\in\mathcal C_t$, showing that
positive coordinate-wise scaling can destroy conflict-freeness.

Exact optimizer maps may also be nonlinear in $a_t$, as occurs with adaptive
scaling. Appendix~\ref{app:nonlinear_preservation} provides a general
nonlinear characterization, with
Proposition~\ref{prop:conflict_preserving} recovered as the affine special
case. Thus, historical state, adaptive scaling, preconditioning, and weight
decay can all transform a conflict-free $a_t$ into an optimizer proposal
outside $\mathcal C_t$.

\subsection{Gradient--Update Alignment}
\label{sec:gua}

\begin{figure}[t]
\begin{center}
\includegraphics[width=\linewidth]{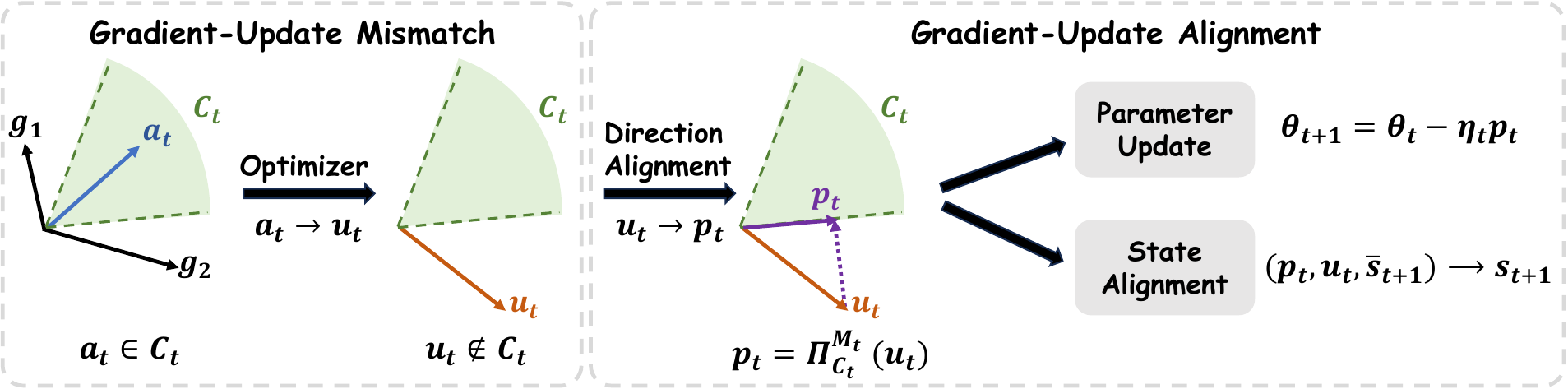}
\end{center}
\caption{Illustration of Gradient--Update Mismatch and Gradient--Update Alignment.
GUA projects the optimizer proposal $u_t$ onto the conflict-free cone
$\mathcal C_t$ to obtain the aligned update $p_t$, applies $p_t$ to the
parameters, and softly aligns the optimizer state with the applied update.}
\label{fig:gua_framework}
\end{figure}

\paragraph{Direction alignment.}
The preceding analysis suggests that conflict handling should extend from the constructed direction to the optimizer proposal, where conflicts directly determine the applied parameter update. Given $u_t$, GUA
projects it onto $\mathcal C_t$ under a positive-definite metric $M_t\succ0$:
\begin{equation}
    p_t
    =
    \Pi_{\mathcal C_t}^{M_t}(u_t)
    =
    \arg\min_{d\in\mathcal C_t}
    \frac12\|d-u_t\|_{M_t}^2.
    \label{eq:gua_projection}
\end{equation}
The projection is solved in the $m$-dimensional dual. For the PINN experiments,
where $m\in\{2,3\}$, we use exact active-set enumeration. For larger task
counts, we solve the nonnegative dual quadratic program. We
use the diagonal Adam as the projection metric. Solver and metric details are provided in
Appendix~\ref{app:projection_solver}, with the metric comparison reported in
Appendix~\ref{app:projection_metric_diagnostic}.
The aligned update is applied as
$\theta_{t+1}=\theta_t-\eta_t p_t$.
By construction, $p_t\in\mathcal C_t$, and $p_t=u_t$ whenever
$u_t\in\mathcal C_t$. Thus, GUA leaves conflict-free proposals unchanged and
projects conflicting proposals onto the conflict-free cone.
To characterize this correction, define the one-step loss change
$
\Delta_{i,t}(d;\eta_t)
=
\mathcal L_i(\theta_t-\eta_t d)-\mathcal L_i(\theta_t),
$
and the conflict violation
$v_t(d)=\max_i[-\langle g_{i,t},d\rangle]_+$.

\begin{proposition}[Local safety and minimum-change projection]
\label{prop:gua_local_guarantee}
Suppose each loss $\mathcal L_i$ is locally smooth near $\theta_t$ and
$u_t\notin\mathcal C_t$. As $\eta_t\to0^+$,
$\left[\max_i\Delta_{i,t}(p_t;\eta_t)\right]_+
=O(\eta_t^2)$ and
$\max_i\Delta_{i,t}(u_t;\eta_t)
\geq \eta_t v_t(u_t)-O(\eta_t^2)$.
Consequently, for all sufficiently small $\eta_t>0$,
$\max_i\Delta_{i,t}(p_t;\eta_t)
< \max_i\Delta_{i,t}(u_t;\eta_t)$.
Moreover, $p_t$ is the closest conflict-free update to $u_t$ under the
$M_t$-metric.
\end{proposition}

For a conflicting optimizer proposal, $u_t$ incurs a positive first-order term
for at least one loss, whereas $p_t$ does not. GUA therefore yields the closest
conflict-free update to $u_t$ under the chosen metric. Full bounds and proofs
are provided in Appendix~\ref{app:proof_gua_local_guarantee}. Q3 tests the
resulting direction alignment mechanism using same-state, norm-matched
comparisons.

\paragraph{State alignment.}

Direction alignment resolves conflicts in the current optimizer proposal by
projecting it onto the conflict-free cone, while state alignment reduces the
persistence of projection-induced corrections in the optimizer memory. Without
state adjustment, components removed by the projection may still affect future
proposals through stored statistics. We therefore softly align the optimizer
state toward targets reconstructed from the applied update $p_t$ only when the
projection modifies the proposal. Otherwise, the tentative state is retained.
Taking Adam as an example, let
\(\bar s_{t+1}=(\bar m_{t+1},\bar v_{t+1})\) denote the tentative state associated
with \(u_t\). We reconstruct a target state from a pseudo-gradient consistent
with \(p_t\) under the current adaptive denominator, as detailed in
Appendix~\ref{app:optimizer_state_alignment}. The persistent state is then
softly aligned as
\begin{equation}
\begin{alignedat}{2}
m_{t+1}
&=
(1-\rho_m)\bar m_{t+1}
+\rho_m\tilde m_{t+1},
&\qquad
v_{t+1}
&=
(1-\rho_v)\bar v_{t+1}
+\rho_v\tilde v_{t+1}.
\end{alignedat}
\label{eq:state_alignment_generic}
\end{equation}
The coefficients $\rho_m,\rho_v\in[0,1]$ control the alignment strength.
Setting $\rho_m=\rho_v=0$ recovers direction alignment alone, while larger
values place greater weight on the target moments reconstructed from $p_t$.
The same principle can extend to other stateful optimizers, including RMSProp,
AdamW, and SOAP, using their respective state-update rules. The corresponding projection and state-alignment procedures, together with the cross-optimizer results, are provided in Appendix~\ref{app:gua_cross_optimizer}.

\begin{table*}[b]
\centering
\caption{Optimizer-induced GUM across PINN benchmarks. The sequence \(R_g \rightarrow R_a \rightarrow R_u\) tracks conflict rates ($\downarrow$) from raw loss-specific gradients through gradient surgery to optimizer transformation. Additional PDE results are provided in Appendix~\ref{app:additional_optimizer_gum}.}
\label{tab:optimizer_gum}
\scriptsize
\setlength{\tabcolsep}{1.35pt}
\renewcommand{\arraystretch}{1.05}
\begin{tabular}{c *{3}{w{c}{16pt}} @{\hspace{6pt}} *{3}{w{c}{16pt}} @{\hspace{6pt}} *{3}{w{c}{16pt}} @{\hspace{6pt}} *{3}{w{c}{16pt}} @{\hspace{6pt}} *{3}{w{c}{16pt}} @{\hspace{6pt}} *{3}{w{c}{16pt}}}
\toprule
\multicolumn{1}{c}{\multirow{3}{*}[-0.5ex]{\textbf{Optimizer}}}
& \multicolumn{6}{c}{\textbf{Schr\"odinger}}
& \multicolumn{6}{c}{\textbf{Burgers}}
& \multicolumn{6}{c}{\textbf{Heat-MS}} \\
\cmidrule(lr){2-7} \cmidrule(lr){8-13} \cmidrule(lr){14-19}
& \multicolumn{3}{c}{\textbf{2-loss}} & \multicolumn{3}{c}{\textbf{3-loss}} & \multicolumn{3}{c}{\textbf{2-loss}} & \multicolumn{3}{c}{\textbf{3-loss}} & \multicolumn{3}{c}{\textbf{2-loss}} & \multicolumn{3}{c}{\textbf{3-loss}} \\
\cmidrule(lr){2-4} \cmidrule(lr){5-7} \cmidrule(lr){8-10} \cmidrule(lr){11-13} \cmidrule(lr){14-16} \cmidrule(lr){17-19}
& $R_g$ & $R_a$ & $R_u$ & $R_g$ & $R_a$ & $R_u$ & $R_g$ & $R_a$ & $R_u$ & $R_g$ & $R_a$ & $R_u$ & $R_g$ & $R_a$ & $R_u$ & $R_g$ & $R_a$ & $R_u$ \\
\midrule
\rowcolor{gray!22}\multicolumn{19}{c}{\textbf{First-order optimizers}} \\
SGD & 100.0 & 0.0 & 0.0 & 100.0 & 0.0 & 0.0 & 99.7 & 0.0 & 0.0 & 99.9 & 0.0 & 0.0 & 70.1 & 0.0 & 0.0 & 99.0 & 0.0 & 0.0 \\
M-SGD & 99.8 & 0.0 & 82.3 & 99.8 & 0.0 & 81.6 & 99.9 & 0.0 & 80.5 & 100.0 & 0.0 & 86.3 & 58.6 & 0.0 & 51.3 & 95.5 & 0.0 & 72.7 \\
RMSProp & 78.7 & 0.0 & 1.8 & 83.3 & 0.0 & 18.8 & 58.4 & 0.0 & 6.1 & 87.1 & 0.0 & 24.7 & 27.7 & 0.0 & 0.0 & 72.8 & 0.0 & 5.5 \\
Adam & 40.8 & 0.0 & 23.7 & 56.3 & 0.0 & 35.2 & 30.5 & 0.0 & 29.4 & 73.7 & 0.0 & 57.4 & 22.8 & 0.0 & 12.8 & 81.3 & 0.0 & 64.6 \\
AdamW & 46.5 & 0.0 & 26.3 & 62.0 & 0.0 & 35.9 & 30.1 & 0.0 & 29.9 & 74.0 & 0.0 & 55.5 & 25.0 & 0.0 & 15.6 & 80.7 & 0.0 & 63.8 \\
\midrule
\rowcolor{gray!22}\multicolumn{19}{c}{\textbf{Second-order optimizers}} \\
SophiaG & 91.7 & 0.0 & 36.3 & 96.3 & 0.0 & 69.5 & 76.9 & 0.0 & 42.8 & 96.9 & 0.0 & 66.4 & 33.9 & 0.0 & 42.4 & 80.7 & 0.0 & 73.2 \\
AdaHessian & 99.7 & 0.0 & 27.2 & 99.7 & 0.0 & 16.9 & 99.4 & 0.0 & 43.9 & 99.9 & 0.0 & 70.9 & 75.2 & 0.0 & 40.9 & 99.8 & 0.0 & 57.1 \\
SOAP & 25.3 & 0.0 & 17.3 & 63.3 & 0.0 & 37.6 & 25.5 & 0.0 & 27.5 & 68.1 & 0.0 & 57.2 & 33.9 & 0.0 & 14.4 & 64.9 & 0.0 & 45.1 \\
\bottomrule
\end{tabular}
\end{table*}

\section{Experiments}
\label{sec:experiments}
\subsection{Experimental Setup}

We structure our experiments around the following research questions:
\textbf{Q1 (Optimizer-Induced GUM):} Does optimizer transformation induce conflicts after gradient surgery?
\textbf{Q2 (PINN Performance):} Does GUA achieve conflict-free applied updates and improve PINN performance?
\textbf{Q3 (Direction Alignment Mechanism):} Does GUA improve optimization primarily through direction alignment rather than changes in update magnitude?
\textbf{Q4 (State Alignment Mechanism):} Does optimizer-state alignment provide
additional benefits beyond direction alignment?
\textbf{Q5 (GUA Overhead):} What computational and memory overhead does GUA
introduce?
\textbf{Q6 (Task-Cardinality Scaling):} Does GUA remain effective as the number
of jointly optimized tasks increases?

For the PINN evaluation, we use six PDE benchmarks following
PINNacle~\cite{hao2024pinnacle} and ConFIG~\cite{liu2025config}, and evaluate
GUA with six representative gradient surgery methods:
PCGrad~\cite{yu2020gradient}, CAGrad~\cite{liu2021conflict},
IMTL-G~\cite{liu2021towards}, A-MTL~\cite{senushkin2023independent},
UPGrad~\cite{quinton2024jacobian}, and ConFIG~\cite{liu2025config}.
PINN results are reported over five independent runs on a single NVIDIA RTX 4090
GPU. We also use CelebA~\cite{liu2015deep} as a controlled task-cardinality
benchmark with varying numbers of jointly optimized tasks, reporting results
over three independent runs.
For the PINN benchmarks, let $\mathcal L_N$, $\mathcal L_B$, and
$\mathcal L_I$ denote the PDE residual, boundary, and initial-condition losses,
respectively. In the 2-loss setting, we introduce a composite loss
$\mathcal L_{BI}$ that aggregates the contributions from the boundary and
initial conditions, and evaluate $[\mathcal L_N,\mathcal L_{BI}]$.  We also evaluate the 3-loss setting
$[\mathcal L_N,\mathcal L_B,\mathcal L_I]$.
Detailed experimental settings are provided in Appendix~\ref{app:pinn_details} and
Appendix~\ref{app:mtl_details}.

\subsection{Q1: Optimizer-Induced GUM}

\textbf{Question Q1:} Does optimizer transformation induce conflicts after gradient surgery?
We use ConFIG as the base gradient surgery method because it constructs
conflict-free directions before optimizer transformation, yielding $R_a=0$.
This isolates the effect of optimizer transformation, since any subsequent
conflict in $u_t$ cannot be attributed to residual conflict in $a_t$. We then
measure the post-optimizer conflict rate $R_u$ across first-order and
second-order optimizer families to assess whether the conflict-free geometry is
preserved after optimizer transformation.

\textbf{Answer to Q1.}
Table~\ref{tab:optimizer_gum} reports the mean conflict rates along the three-stage transition
$R_g \rightarrow R_a \rightarrow R_u$, corresponding to the raw
loss-specific gradients $\{g_{i,t}\}_{i=1}^m$, constructed direction $a_t$,
and optimizer proposal $u_t$, respectively. ConFIG yields $R_a=0$ in all
settings, so nonzero $R_u$ directly indicates GUM.
Consistent with the conflict-preservation analysis in
Section~\ref{sec:conflict_preservation}, plain SGD preserves $R_u=0$.
In contrast, momentum alone produces substantial GUM, with M-SGD yielding
$R_u=51.3\%$--$86.3\%$. Adaptive and curvature-based optimizers can also produce
nonzero $R_u$, with the rate varying across PDEs. These results show that
optimizer transformation can reintroduce conflicts after gradient surgery,
depending on its interaction with the current loss-gradient geometry.

\subsection{Q2: PINN Performance}

\textbf{Question Q2:} Does GUA achieve conflict-free applied updates and improve PINN performance?
After establishing GUM, we assess whether GUA achieves conflict-free applied
updates and improves PINN performance. We evaluate GUA on six PDE benchmarks using representative gradient surgery methods, with Adam as a
no-surgery control that passes the original aggregate gradient directly to the
optimizer. Figure~\ref{fig:pinn_main_results} reports
the relative $L_2$ error.  Quantitative results and detailed conflict
transitions
$R_g\!\rightarrow R_a\!\rightarrow R_u\!\rightarrow R_p$
are provided in Appendix~\ref{app:full_pinn_results}.

\textbf{Answer to Q2.}
As shown in Appendix~\ref{app:full_pinn_results}, GUA eliminates update-level
conflicts, yielding \(R_p=0\) across all PINN settings. Figure~\ref{fig:pinn_main_results}
shows that GUA consistently improves PINN accuracy across gradient surgery methods
and PDE benchmarks. The mean relative \(L_2\) error decreases in every
setting, with individual reductions ranging from \(11.6\%\) to \(98.2\%\) and median
reductions of \(42.2\%\) and \(61.9\%\) in the two- and three-loss settings,
respectively. Notably, Adam+GUA, without gradient surgery, improves over Adam in all 10 PINN settings and outperforms all standalone gradient surgery baselines in 4 settings, highlighting the value of conflict handling after optimizer transformation. GUA also remains effective with RMSProp, AdamW, and SOAP, as shown in
Appendix~\ref{app:gua_cross_optimizer}.

\begin{figure*}[t]
\centering
\includegraphics[width=\textwidth]{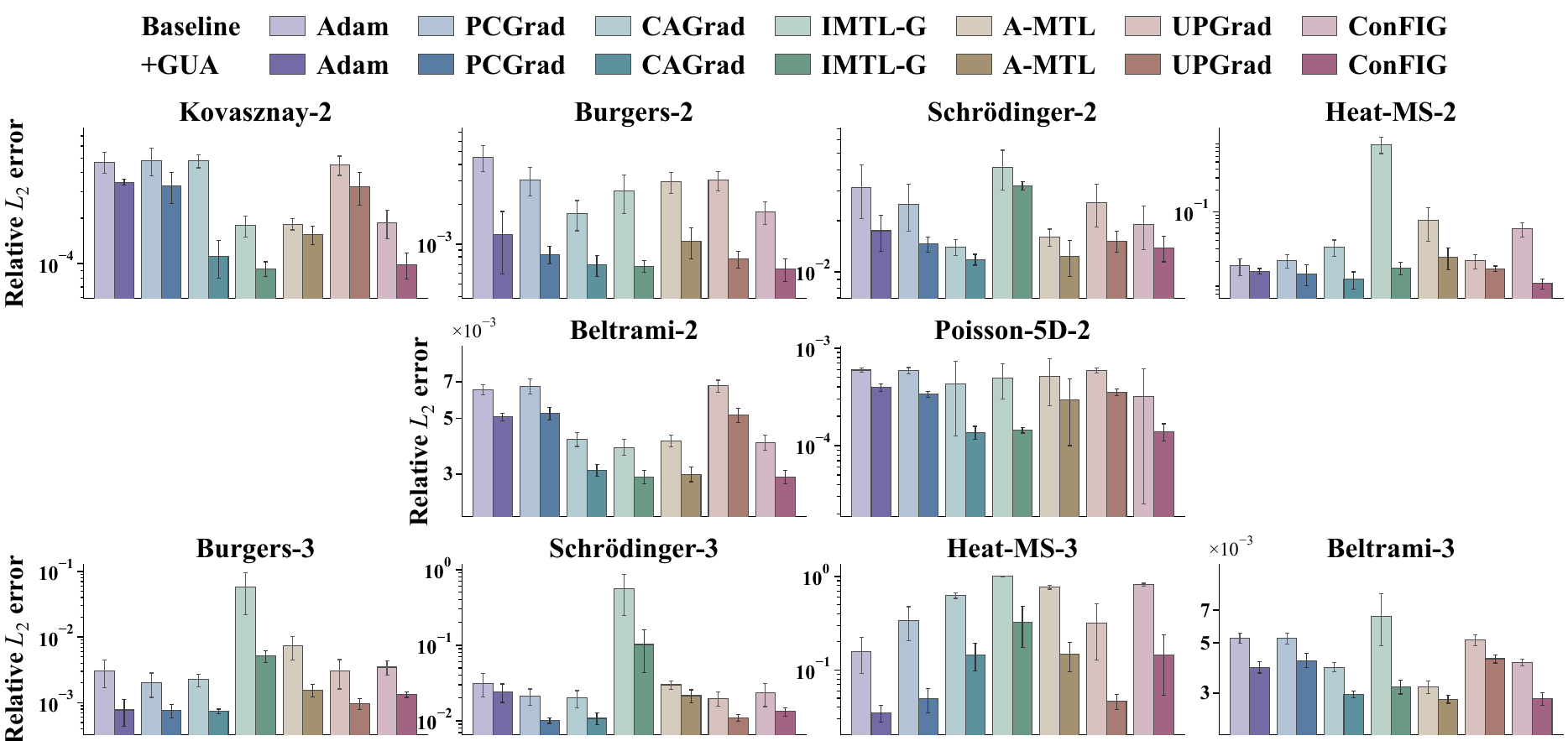}
\caption{Performance of different gradient surgery methods with and without GUA across
PINN benchmarks under 2-loss and 3-loss settings.}
\label{fig:q2_pinn_performance}
\label{fig:pinn_main_results}
\end{figure*}

\subsection{Q3: Direction  Alignment Mechanism}

\textbf{Question Q3:} Does GUA improve optimization primarily through
direction alignment rather than changes in update magnitude?
Direction alignment can change both the direction and norm of the optimizer proposal,
so performance gains alone cannot distinguish their effects. Using ConFIG as the
base gradient surgery method, we conduct a same-state, norm-matched comparison
across all 10 PINN settings. At each evaluated step, we fix the parameters,
optimizer state, batch, and learning rate, and compare $u_t$, $p_t$, and
$
q_t=\frac{\|p_t\|_2}{\|u_t\|_2}u_t.
$
Here, $q_t$ preserves the direction of $u_t$ while matching the norm of $p_t$,
thereby isolating the effect of direction alignment. We evaluate the worst
relative one-step loss change
$
W_t^{\mathrm{rel}}(d)
=
\max_i
\frac{
\mathcal L_i(\theta_t-\eta_t d)-\mathcal L_i(\theta_t)}
{|\mathcal L_i(\theta_t)|+\epsilon}.
$
The positive, step-fixed denominator rescales each loss without changing
$\mathcal C_t$, so the same local-safety argument applies to
$W_t^{\mathrm{rel}}$.

\textbf{Answer to Q3.}
As shown in Figure~\ref{fig:mechanism}, the benefit of GUA is primarily driven
by direction alignment. Across all 10 PINN settings, $p_t$ yields a lower
worst relative one-step loss change than $u_t$ in $95.1\%$ of paired same-state
comparisons and than the norm-matched $q_t$ in $93.4\%$ of cases
(see Appendix Table~\ref{tab:q3_equation_win_rates}). Since $q_t$ matches the norm
of $p_t$ while preserving the direction of $u_t$, the improvement over $q_t$
isolates the effect of direction alignment. These results empirically support the direction alignment mechanism
in Proposition~\ref{prop:gua_local_guarantee} and the norm-matching
analysis in Appendix~\ref{app:proof_norm_matching}.

\begin{figure*}[t]
\centering
\includegraphics[width=\textwidth]{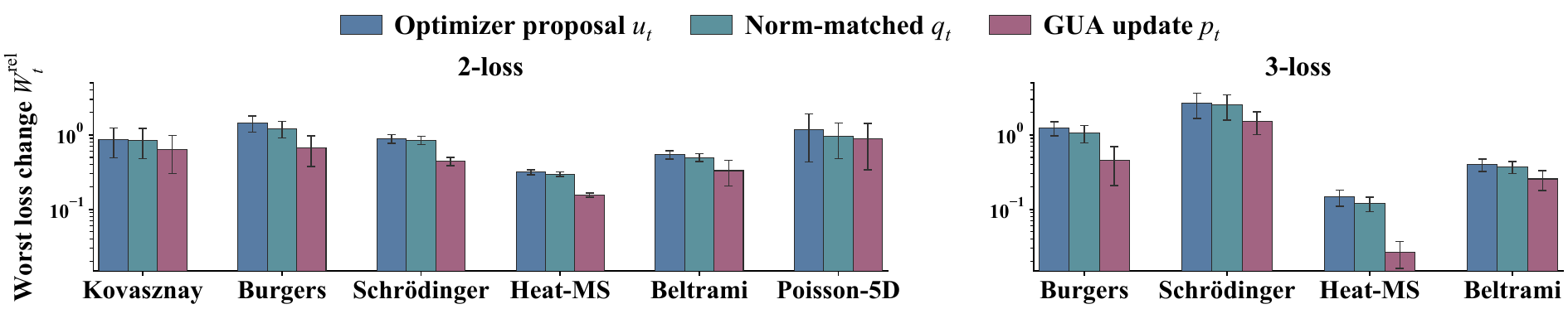}
\caption{Worst relative one-step loss change $W_t^{\mathrm{rel}}$ for the optimizer
proposal $u_t$, the GUA update $p_t$, and the norm-matched control $q_t$ under
paired same-state evaluations across PINN benchmarks. }
\label{fig:q3_mechanism}
\label{fig:mechanism}
\end{figure*}

\subsection{Q4: State Alignment Mechanism}
\textbf{Question Q4:} Does optimizer-state alignment provide additional
benefits beyond direction alignment?
Direction alignment acts on the current optimizer proposal $u_t$, but stateful
optimizers retain internal statistics that influence future proposals. We
therefore examine whether adjusting the optimizer state toward targets
reconstructed from the applied update $p_t$ provides additional benefits beyond
direction alignment alone. Using ConFIG as the base gradient surgery method
with Adam as the optimizer, we compare the gradient surgery baseline without
GUA, direction alignment only $(\rho_m,\rho_v)=(0,0)$, soft state alignment
(see Appendix~\ref{app:optimizer_state_alignment}),
exact first-moment alignment $(\rho_m,\rho_v)=(1,0)$, and exact alignment of
both Adam moments $(\rho_m,\rho_v)=(1,1)$.

\textbf{Answer to Q4.}
Figure~\ref{fig:state_alignment} shows that direction alignment alone already improves performance, while soft state alignment provides further improvements in most settings. In contrast, exact alignment is less effective,
and aligning both Adam moments exactly degrades performance in some settings
and becomes numerically unstable in one Heat-MS setting.
A possible explanation is that direction alignment affects only the current
optimizer proposal, whereas the moments maintained by Adam encode information
accumulated over a longer optimization horizon.
Fully replacing them with states reconstructed from $p_t$ can discard useful
history and alter future adaptive scaling. Soft alignment instead preserves
part of this history while moving the optimizer state toward target moments
reconstructed from $p_t$, motivating its use in GUA.
\begin{figure*}[h]
\centering
\includegraphics[width=\textwidth]{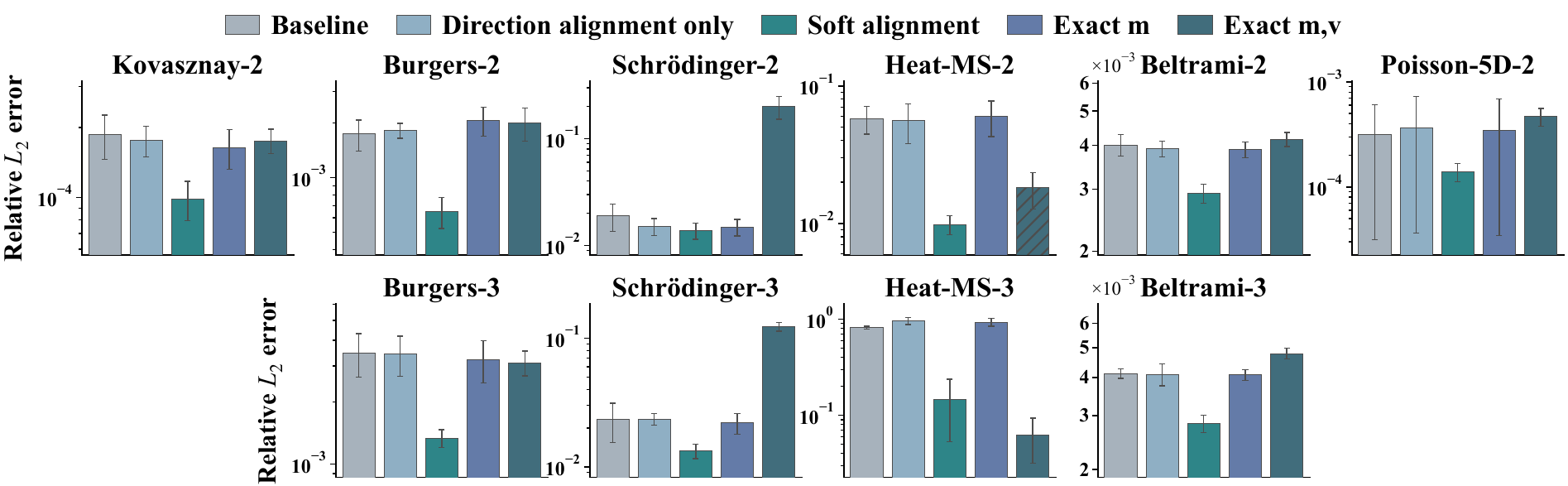}
\caption{State-alignment ablation on PINN benchmarks.}
\label{fig:state_alignment}
\end{figure*}

\subsection{Q5: GUA Overhead}

\textbf{Question Q5:} What computational and memory overhead does GUA
introduce?
We evaluate the practical overhead of GUA by comparing training time and peak
GPU memory with and without GUA across all PINN settings, using ConFIG as the
base method. GUA performs additional conflict checking and direction alignment
at each step. The projection is solved in the $m$-dimensional dual without
forming a $p\times p$ matrix. Its parameter-dimensional cost is $O(m^2p)$ for
Gram-matrix construction, followed by a small dual solve. Since
$m\in\{2,3\}$ in the PINN experiments, this dual problem is solved exactly by
active-set enumeration. Further details are provided in
Appendix~\ref{app:projection_solver}.

\textbf{Answer to Q5.}
As shown in Figure~\ref{fig:overhead}, GUA introduces additional training-time
overhead while adding little peak GPU memory overhead. Across all PINN
settings, GUA consistently improves accuracy. These results highlight a trade-off
between accuracy and efficiency. GUA improves
accuracy at the cost of additional computation while incurring little
additional memory usage.

\begin{figure*}[t]
\centering
\includegraphics[width=\textwidth]{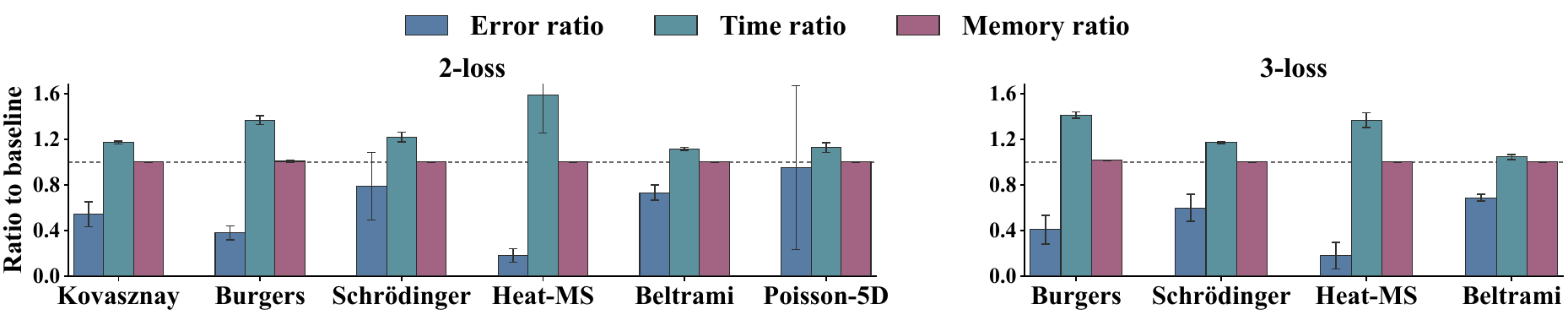}
\caption{Relative $L_2$ error ($\downarrow$), training-time ($\downarrow$), and peak-memory ($\downarrow$) ratios across PINN settings.}
\label{fig:overhead}
\end{figure*}

\subsection{Q6: Task-Cardinality Scaling}

\textbf{Question Q6:} Does GUA remain effective as the number of jointly
optimized tasks increases?
We use CelebA~\cite{liu2015deep} as a controlled task-cardinality benchmark,
varying the number of attribute tasks from 2 to 40 while fixing ConFIG as the
baseline. For each task count, we compare ConFIG and
ConFIG+GUA in terms of  $\overline{F_1}$, training time, and peak GPU
memory. The projection uses exact active-set enumeration for small $m$ and an $m$-dimensional nonnegative dual QP for larger $m$. Further details are provided in
Appendix~\ref{app:projection_solver}.

\begin{figure*}[t]
\centering
\includegraphics[width=\textwidth]{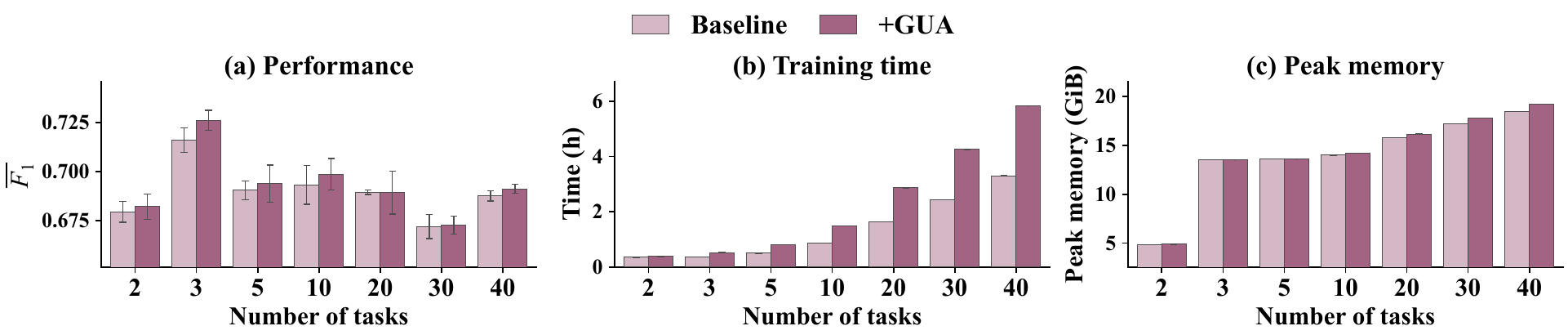}
\caption{Task scaling of GUA in terms
of $\overline{F_1}$ ($\uparrow$), training time ($\downarrow$), and peak GPU memory ($\downarrow$).}
\label{fig:celeba_task_scaling}
\end{figure*}

\textbf{Answer to Q6.}
Figure~\ref{fig:celeba_task_scaling} shows that GUA remains effective as task
cardinality increases from 2 to 40, improving the $\overline{F_1}
$ in six of
seven settings and achieving comparable performance at 20 tasks. The additional
memory cost remains small as task cardinality increases, reaching only 4.2\% at
40 tasks, while training-time overhead becomes more pronounced for larger task
counts. These results demonstrate that GUA remains effective as the
number of jointly optimized tasks increases.

\section{Conclusion}

This work provides a new perspective on gradient conflict handling in PINN training.
We identify the largely overlooked phenomenon of GUM and demonstrate its prevalence
across evaluated PINN settings through theoretical analysis and extensive experiments.
Motivated by this finding, we propose GUA to enforce conflict-free geometry at the
update level. Extensive experiments show that GUA effectively eliminates update-level
conflicts and substantially improves PINN performance.
\paragraph{Limitations and Future Work.}
Despite these advances, the current formulation of GUM and GUA is based on local first-order gradient geometry and therefore does not explicitly capture curvature, finite-step, or longer-horizon optimization effects. Moreover, the current hard conflict-free criterion may be conservative by excluding update directions that could become beneficial beyond the local first-order view. Therefore, extending the framework beyond instantaneous first-order geometry, together with developing softer conflict criteria, is an important direction for future work.

\bibliography{references}
\bibliographystyle{arxiv_preprint}
\clearpage
\appendix
\section{Theory and Proofs}
\label{app:theory_proofs}

% =========================================================
\subsection{Proof of Proposition~\ref{prop:conflict_preserving}}
\label{app:proof_conflict_preserving}

\paragraph{Proposition~\ref{prop:conflict_preserving}.}
Let \(G_t=[g_{1,t},\ldots,g_{m,t}]\) and
\[
\mathcal C_t
=
\{d\in\mathbb R^p\mid G_t^\top d\ge0\}.
\]
The affine transformation \(T_t(a)=P_ta+b_t\) is conflict-preserving on
\(\mathcal C_t\), i.e.,
\[
a\in\mathcal C_t
\quad\Longrightarrow\quad
T_t(a)\in\mathcal C_t,
\]
if and only if
\[
b_t\in\mathcal C_t,
\qquad
P_t^\top g_{i,t}
\in
\operatorname{cone}(g_{1,t},\ldots,g_{m,t}),
\quad
i=1,\ldots,m.
\]

\begin{proof}
We first prove necessity. Since \(\mathcal C_t\) is a cone,
\(0\in\mathcal C_t\). If \(T_t\) is conflict-preserving, then
\[
T_t(0)=b_t\in\mathcal C_t.
\]

Now take any \(a\in\mathcal C_t\) and any scalar \(\lambda\ge0\).
Since \(\lambda a\in\mathcal C_t\), conflict preservation gives
\[
T_t(\lambda a)
=
\lambda P_ta+b_t
\in
\mathcal C_t.
\]
Therefore, for every \(i\),
\[
\lambda\langle g_{i,t},P_ta\rangle
+
\langle g_{i,t},b_t\rangle
\ge0
\qquad
\text{for all }\lambda\ge0.
\]
Letting \(\lambda\to\infty\) implies
\[
\langle g_{i,t},P_ta\rangle\ge0
\qquad
\text{for all }a\in\mathcal C_t.
\]
Equivalently,
\[
\langle P_t^\top g_{i,t},a\rangle\ge0
\qquad
\text{for all }a\in\mathcal C_t.
\]

Define the dual cone of \(\mathcal C_t\) as
\[
\mathcal C_t^\ast
=
\left\{
y\in\mathbb R^p
\;\middle|\;
\langle y,d\rangle\ge0
\text{ for all }d\in\mathcal C_t
\right\}.
\]
The preceding inequality shows that
\[
P_t^\top g_{i,t}\in\mathcal C_t^\ast.
\]
Since
\[
\mathcal C_t
=
\{d\in\mathbb R^p\mid G_t^\top d\ge0\}
\]
is a closed polyhedral cone, Farkas' lemma gives
\[
\mathcal C_t^\ast
=
\operatorname{cone}(g_{1,t},\ldots,g_{m,t}).
\]
Hence,
\[
P_t^\top g_{i,t}
\in
\operatorname{cone}(g_{1,t},\ldots,g_{m,t}),
\qquad
i=1,\ldots,m.
\]

We now prove sufficiency. Suppose
\[
b_t\in\mathcal C_t
\]
and
\[
P_t^\top g_{i,t}
\in
\operatorname{cone}(g_{1,t},\ldots,g_{m,t})
\qquad
\text{for all }i.
\]
For any \(a\in\mathcal C_t\),
\[
\langle g_{i,t},P_ta\rangle
=
\langle P_t^\top g_{i,t},a\rangle
\ge0,
\]
because \(P_t^\top g_{i,t}\) is a nonnegative combination of
\(g_{1,t},\ldots,g_{m,t}\), while
\[
\langle g_{j,t},a\rangle\ge0
\qquad
\text{for every }j.
\]
Thus,
\[
P_ta\in\mathcal C_t.
\]
Since \(b_t\in\mathcal C_t\) and \(\mathcal C_t\) is a convex cone, it is
closed under addition. Therefore,
\[
P_ta+b_t\in\mathcal C_t.
\]
Hence,
\[
T_t(a)=P_ta+b_t\in\mathcal C_t
\qquad
\text{for every }a\in\mathcal C_t,
\]
so \(T_t\) is conflict-preserving.
\end{proof}

% =========================================================
\subsection{General Nonlinear Conflict Preservation}
\label{app:nonlinear_preservation}

The fixed-state affine result above provides an interpretable characterization,
but exact optimizer maps need not be affine in the constructed direction. We
therefore give a general characterization for nonlinear optimizer
transformations.

\begin{proposition}[General nonlinear conflict preservation]
\label{prop:nonlinear_conflict_preserving}
Fix \(G_t=[g_{1,t},\ldots,g_{m,t}]\) and all step-\(t\) quantities other
than the input direction \(a\), and let
\[
T_t:\mathbb R^p\rightarrow\mathbb R^p
\]
denote the resulting one-step optimizer transformation. Suppose \(T_t\) is
locally Lipschitz. For each \(a\in\mathcal C_t\), define the ray-wise map
\[
\phi_{t,a}(\tau)
=
T_t(\tau a),
\qquad
\tau\in[0,1].
\]
Then \(T_t\) is conflict-preserving on \(\mathcal C_t\) if and only if,
for every \(a\in\mathcal C_t\) and every \(i=1,\ldots,m\),
\[
\langle g_{i,t},T_t(0)\rangle
+
\int_0^1
\left\langle
g_{i,t},
\phi_{t,a}'(\tau)
\right\rangle
\,d\tau
\ge0,
\]
where \(\phi_{t,a}'(\tau)\) exists almost everywhere.

If \(T_t\) is continuously differentiable, this condition becomes
\[
\langle g_{i,t},T_t(0)\rangle
+
\int_0^1
\left\langle
J_{T_t}(\tau a)^\top g_{i,t},
a
\right\rangle
\,d\tau
\ge0,
\]
where \(J_{T_t}\) denotes the Jacobian of \(T_t\).
\end{proposition}

\begin{proof}
For any fixed \(a\in\mathcal C_t\), the locally Lipschitz property of
\(T_t\) implies that the ray-wise map
\[
\phi_{t,a}(\tau)=T_t(\tau a)
\]
is absolutely continuous on \([0,1]\). Hence it is differentiable almost
everywhere, and the fundamental theorem for absolutely continuous functions
gives
\[
T_t(a)
=
T_t(0)
+
\int_0^1
\phi_{t,a}'(\tau)
\,d\tau.
\]
Taking the inner product with \(g_{i,t}\) yields
\[
\langle g_{i,t},T_t(a)\rangle
=
\langle g_{i,t},T_t(0)\rangle
+
\int_0^1
\left\langle
g_{i,t},
\phi_{t,a}'(\tau)
\right\rangle
\,d\tau.
\]
By definition,
\[
T_t(a)\in\mathcal C_t
\quad\Longleftrightarrow\quad
\langle g_{i,t},T_t(a)\rangle\ge0
\quad
\text{for all }i.
\]
Therefore, requiring the displayed inequality for every
\(a\in\mathcal C_t\) and every \(i\) is equivalent to
\[
T_t(\mathcal C_t)\subseteq\mathcal C_t.
\]

If \(T_t\) is continuously differentiable, the chain rule gives
\[
\phi_{t,a}'(\tau)
=
J_{T_t}(\tau a)a,
\]
which yields the stated Jacobian form.
\end{proof}

\paragraph{Relation to the affine case.}
For an affine transformation
\[
T_t(a)=P_ta+b_t,
\]
we have
\[
T_t(0)=b_t,
\qquad
J_{T_t}(a)=P_t.
\]
Thus, the nonlinear characterization reduces to
\[
\langle g_{i,t},b_t\rangle
+
\langle P_t^\top g_{i,t},a\rangle
\ge0
\qquad
\text{for every }a\in\mathcal C_t.
\]
Proposition~\ref{prop:conflict_preserving} gives the corresponding
interpretable necessary-and-sufficient conditions
\[
b_t\in\mathcal C_t,
\qquad
P_t^\top g_{i,t}
\in
\operatorname{cone}(g_{1,t},\ldots,g_{m,t}).
\]
Hence Proposition~\ref{prop:conflict_preserving} is the affine specialization
of the general nonlinear characterization.

The locally Lipschitz formulation also covers piecewise-smooth optimizer
transformations: the ray-wise derivative is interpreted almost everywhere,
so nondifferentiabilities introduced, for example, by clipping do not affect
the characterization.

% =========================================================
\subsection{Proof of Proposition~\ref{prop:gua_local_guarantee}}
\label{app:proof_gua_local_guarantee}

\paragraph{Proposition~\ref{prop:gua_local_guarantee}.}
Suppose each loss \(\mathcal L_i\) is \(\beta_i\)-smooth, with
\(\beta_i>0\), on a convex neighborhood of \(\theta_t\), and suppose the
optimizer proposal is conflicting, i.e.,
\[
u_t\notin\mathcal C_t.
\]
Let
\[
\delta_t
=
v_t(u_t)
=
\max_i[-\langle g_{i,t},u_t\rangle]_+
>0
\]
and choose
\[
i^\star
\in
\arg\max_i
[-\langle g_{i,t},u_t\rangle]_+.
\]
Let
\[
p_t
=
\Pi_{\mathcal C_t}^{M_t}(u_t)
=
\arg\min_{d\in\mathcal C_t}
\frac12\|d-u_t\|_{M_t}^2.
\]
Then, for every \(\eta_t>0\) such that
\(\theta_t-\eta_tu_t\) and \(\theta_t-\eta_tp_t\) remain in this
neighborhood,
\[
\max_i\Delta_{i,t}(p_t;\eta_t)
\leq
\frac{\beta_{\max}\eta_t^2}{2}\|p_t\|_2^2,
\]
and
\[
\Delta_{i^\star,t}(u_t;\eta_t)
\geq
\eta_t\delta_t
-
\frac{\beta_{i^\star}\eta_t^2}{2}\|u_t\|_2^2,
\]
where
\[
\beta_{\max}
=
\max_i\beta_i.
\]
Consequently,
\[
\max_i\Delta_{i,t}(p_t;\eta_t)
<
\max_i\Delta_{i,t}(u_t;\eta_t)
\]
for all sufficiently small \(\eta_t>0\). Moreover, for every
\(d\in\mathcal C_t\),
\[
\|u_t-d\|_{M_t}^2
\geq
\|u_t-p_t\|_{M_t}^2
+
\|p_t-d\|_{M_t}^2.
\]

\begin{proof}
Because \(M_t\succ0\) and \(\mathcal C_t\) is a nonempty closed convex cone,
the metric projection
\[
p_t
=
\arg\min_{d\in\mathcal C_t}
\frac12\|d-u_t\|_{M_t}^2
\]
exists and is unique. For convenience, define
\[
W_t(d;\eta_t)
=
\max_{1\leq i\leq m}
\Delta_{i,t}(d;\eta_t).
\]

\paragraph{First-order safety.}
By the definition of \(\mathcal C_t\),
\[
p_t\in\mathcal C_t
\quad\Longrightarrow\quad
\langle g_{i,t},p_t\rangle\geq0,
\qquad
i=1,\ldots,m.
\]
Since \(\mathcal L_i\) is \(\beta_i\)-smooth, for any direction \(d\) and
any admissible step size \(\eta_t>0\),
\[
\left|
\mathcal L_i(\theta_t-\eta_td)
-
\mathcal L_i(\theta_t)
+
\eta_t\langle g_{i,t},d\rangle
\right|
\leq
\frac{\beta_i\eta_t^2}{2}\|d\|_2^2.
\]
Applying the upper bound with \(d=p_t\) gives
\[
\Delta_{i,t}(p_t;\eta_t)
\leq
-\eta_t\langle g_{i,t},p_t\rangle
+
\frac{\beta_i\eta_t^2}{2}\|p_t\|_2^2.
\]
Because
\[
\langle g_{i,t},p_t\rangle\geq0,
\]
we obtain
\[
\Delta_{i,t}(p_t;\eta_t)
\leq
\frac{\beta_i\eta_t^2}{2}\|p_t\|_2^2.
\]
Therefore,
\[
W_t(p_t;\eta_t)
=
\max_i\Delta_{i,t}(p_t;\eta_t)
\leq
\frac{\beta_{\max}\eta_t^2}{2}\|p_t\|_2^2.
\]
Thus, the projected update contains no positive first-order loss-change term.

\paragraph{One-step improvement for conflicting optimizer proposals.}
Since
\[
u_t\notin\mathcal C_t,
\]
we have
\[
\delta_t
=
v_t(u_t)
=
\max_i[-\langle g_{i,t},u_t\rangle]_+
>0.
\]
Choose
\[
i^\star
\in
\arg\max_i[-\langle g_{i,t},u_t\rangle]_+.
\]
Then
\[
-\langle g_{i^\star,t},u_t\rangle
=
\delta_t,
\]
and therefore
\[
\langle g_{i^\star,t},u_t\rangle
=
-\delta_t.
\]

Using the lower smoothness bound with \(d=u_t\) yields
\begin{align*}
\Delta_{i^\star,t}(u_t;\eta_t)
&\geq
-\eta_t\langle g_{i^\star,t},u_t\rangle
-
\frac{\beta_{i^\star}\eta_t^2}{2}\|u_t\|_2^2
\\
&=
\eta_t\delta_t
-
\frac{\beta_{i^\star}\eta_t^2}{2}\|u_t\|_2^2.
\end{align*}
Since \(W_t(u_t;\eta_t)\) is the maximum loss change,
\[
W_t(u_t;\eta_t)
\geq
\Delta_{i^\star,t}(u_t;\eta_t),
\]
and hence
\[
W_t(u_t;\eta_t)
\geq
\eta_t\delta_t
-
\frac{\beta_{i^\star}\eta_t^2}{2}\|u_t\|_2^2.
\]

Combining the bounds for \(p_t\) and \(u_t\), we obtain
\[
W_t(p_t;\eta_t)
<
W_t(u_t;\eta_t)
\]
whenever
\[
\frac{\beta_{\max}\eta_t^2}{2}\|p_t\|_2^2
<
\eta_t\delta_t
-
\frac{\beta_{i^\star}\eta_t^2}{2}\|u_t\|_2^2.
\]
For \(\eta_t>0\), this is equivalent to
\[
0<\eta_t<
\frac{2\delta_t}
{
\beta_{i^\star}\|u_t\|_2^2
+
\beta_{\max}\|p_t\|_2^2
}.
\]
Therefore, whenever the two trial points remain in the assumed neighborhood,
\[
W_t(p_t;\eta_t)
<
W_t(u_t;\eta_t)
\]
for every
\[
0<\eta_t<
\frac{2\delta_t}
{
\beta_{i^\star}\|u_t\|_2^2
+
\beta_{\max}\|p_t\|_2^2
}.
\]
In particular, the strict inequality holds for all sufficiently small
admissible step sizes.

\paragraph{Least-change correction.}
The first-order optimality condition of the metric projection is
\[
\langle M_t(p_t-u_t),d-p_t\rangle
\geq0,
\qquad
\forall d\in\mathcal C_t.
\]
Define the \(M_t\)-inner product as
\[
\langle x,y\rangle_{M_t}
=
x^\top M_ty.
\]
The optimality condition is equivalently
\[
\langle u_t-p_t,d-p_t\rangle_{M_t}
\leq0,
\]
and hence
\[
\langle u_t-p_t,p_t-d\rangle_{M_t}
\geq0.
\]

For any \(d\in\mathcal C_t\), expanding the squared distance gives
\begin{align*}
\|u_t-d\|_{M_t}^2
&=
\|u_t-p_t\|_{M_t}^2
+
\|p_t-d\|_{M_t}^2
+
2\langle u_t-p_t,p_t-d\rangle_{M_t}
\\
&\geq
\|u_t-p_t\|_{M_t}^2
+
\|p_t-d\|_{M_t}^2.
\end{align*}
This proves the metric projection inequality.

In particular, if
\[
a_t\in\mathcal C_t,
\]
then substituting \(d=a_t\) gives
\[
\|u_t-a_t\|_{M_t}^2
\geq
\|u_t-p_t\|_{M_t}^2
+
\|p_t-a_t\|_{M_t}^2,
\]
or equivalently,
\[
\|p_t-a_t\|_{M_t}^2
\leq
\|u_t-a_t\|_{M_t}^2
-
\|u_t-p_t\|_{M_t}^2.
\]
Thus, GUA restores feasibility through the smallest feasible correction under
the \(M_t\)-metric. More generally, the projection inequality shows that
\(p_t\) is closer than \(u_t\) to every feasible direction
\(d\in\mathcal C_t\), including the original constructed direction \(a_t\)
whenever \(a_t\in\mathcal C_t\).
\end{proof}

If \(\langle g_{i,t},p_t\rangle>0\) for every \(i\), the same smoothness
bounds further imply simultaneous decrease of all losses for sufficiently
small step sizes. When \(p_t\) lies on the boundary of \(\mathcal C_t\),
however, some inner products may be zero, and
Proposition~\ref{prop:gua_local_guarantee} guarantees only the absence of a
positive first-order loss-change term for the corresponding losses.

% =========================================================
\subsection{Norm Matching Does Not Restore Conflict-Freeness}
\label{app:proof_norm_matching}

The following observation provides the theoretical basis for the norm-matched
control used in Q3 and Figure~\ref{fig:mechanism}.

\paragraph{Norm-matching observation.}
Suppose
\[
u_t\notin\mathcal C_t
\]
and let
\[
q_t=\alpha_tu_t
\]
for any \(\alpha_t>0\). Then
\[
v_t(q_t)
=
\alpha_tv_t(u_t)
>0.
\]
Thus, positive rescaling cannot restore conflict-freeness. In particular,
when \(p_t\neq0\), the norm-matched control
\[
q_t
=
\frac{\|p_t\|_2}{\|u_t\|_2}u_t
\]
satisfies
\[
\|q_t\|_2=\|p_t\|_2
\]
but remains outside \(\mathcal C_t\).

Indeed, since
\[
u_t\notin\mathcal C_t,
\]
we have
\[
v_t(u_t)>0.
\]
For any \(\alpha_t>0\),
\begin{align*}
v_t(\alpha_tu_t)
&=
\max_i[-\langle g_{i,t},\alpha_tu_t\rangle]_+
\\
&=
\alpha_t
\max_i[-\langle g_{i,t},u_t\rangle]_+
\\
&=
\alpha_tv_t(u_t)
>0.
\end{align*}
Hence
\[
\alpha_tu_t\notin\mathcal C_t
\qquad
\text{for every }\alpha_t>0.
\]
When \(p_t\neq0\), the norm-matching factor
\[
\alpha_t
=
\frac{\|p_t\|_2}{\|u_t\|_2}
\]
is strictly positive, because \(u_t\notin\mathcal C_t\) implies
\(u_t\neq0\). Therefore,
\[
\|q_t\|_2=\|p_t\|_2,
\qquad
v_t(q_t)
=
\frac{\|p_t\|_2}{\|u_t\|_2}
v_t(u_t)
>0.
\]
Thus,
\[
q_t\notin\mathcal C_t,
\]
even though \(q_t\) and \(p_t\) have the same Euclidean norm. Hence, norm matching alone cannot remove GUM, and changing the update direction is necessary to restore conflict-freeness.

\newpage
\section{Experimental Details for PINNs}
\label{app:pinn_details}

\subsection{PDE Benchmarks}
\label{app:pde_benchmarks}

We evaluate GUA on standard PDE benchmarks spanning steady and time-dependent PDEs, scalar and multi-component systems, low- and high-dimensional domains, and incompressible flows. The benchmark protocols
are mainly based on PINNacle~\cite{hao2024pinnacle} and
ConFIG~\cite{liu2025config}.

\paragraph{Schr\"odinger.}
The Schr\"odinger benchmark solves a nonlinear Schr\"odinger equation by representing the complex solution as $h=u+iv$. The residuals are
\begin{align*}
u_t + \frac{1}{2}v_{xx} + (u^2+v^2)v &= 0,\\
v_t - \frac{1}{2}u_{xx} - (u^2+v^2)u &= 0,
\end{align*}
on $(x,t)\in[-5,5]\times[0,\pi/2]$. The initial condition is
\begin{equation*}
u(x,0)=\frac{2}{\cosh(x)},\qquad v(x,0)=0.
\end{equation*}
The equation follows periodic boundary conditions at $x=-5$ and $x=5$.

\paragraph{Burgers.}
The Burgers benchmark solves the one-dimensional viscous Burgers equation on $(x,t)\in[-1,1]\times[0,1]$:
\begin{equation*}
u_t + u u_x - \nu u_{xx}=0,
\qquad
\nu=\frac{0.01}{\pi}.
\end{equation*}
The initial condition is
\begin{equation*}
u(x,0)=-\sin(\pi x),
\end{equation*}
and homogeneous Dirichlet boundary conditions are imposed:
\begin{equation*}
u(-1,t)=u(1,t)=0.
\end{equation*}

\paragraph{Kovasznay flow.}
The Kovasznay benchmark solves a steady two-dimensional incompressible-flow problem on
$\Omega=[-0.5,1]\times[-0.5,1.5]$:
\begin{align*}
u u_x + v u_y + p_x - \nu (u_{xx}+u_{yy}) &= 0,\\
u v_x + v v_y + p_y - \nu (v_{xx}+v_{yy}) &= 0,\\
u_x+v_y &= 0.
\end{align*}
The Reynolds number is $\mathrm{Re}=40$, so $\nu=1/\mathrm{Re}$. The analytical solution is
\begin{align*}
u^\star(x,y) &= 1-\exp(\lambda x)\cos(2\pi y),\\
v^\star(x,y) &= \frac{\lambda}{2\pi}\exp(\lambda x)\sin(2\pi y),\\
p^\star(x,y) &= \frac{1}{2}\left(1-\exp(2\lambda x)\right),
\end{align*}
where
\begin{equation*}
\lambda=\frac{1}{2\nu}-\sqrt{\frac{1}{4\nu^2}+4\pi^2}.
\end{equation*}
Dirichlet boundary conditions are given by the analytical solution on $\partial\Omega$. The same analytical solution is used to generate validation and test data.

\paragraph{Heat-MS.}
The Heat-MS reference data follow a two-dimensional multiscale heat solution on $(x,y,t)\in[0,1]\times[0,1]\times[0,5]$. The analytical reference corresponds to
\begin{equation*}
u_t - c_x u_{xx} - c_y u_{yy}=0,
\end{equation*}
with
\begin{equation*}
c_x=\frac{1}{(500\pi)^2}, \qquad c_y=\frac{1}{\pi^2}.
\end{equation*}
The initial condition is
\begin{equation*}
u(x,y,0)=\sin(20\pi x)\sin(\pi y),
\end{equation*}
and homogeneous Dirichlet boundary conditions are imposed on the spatial boundary. The exact solution used for evaluation is
\begin{equation*}
u^\star(x,y,t)
=
\sin(20\pi x)\sin(\pi y)
\exp\left[-\left(c_x(20\pi)^2+c_y\pi^2\right)t\right].
\end{equation*}
The training residual uses the same anisotropic diffusivities,
\begin{equation*}
u_t-c_xu_{xx}-c_yu_{yy}=0,
\end{equation*}
as the analytical reference. The initial, boundary, validation, and test data
are sampled from that reference solution.

\paragraph{Beltrami flow.}
The Beltrami benchmark solves an unsteady three-dimensional incompressible-flow problem on $(x,y,z,t)\in[-1,1]^3\times[0,1]$:
\begin{align*}
u_t + u u_x + v u_y + w u_z + p_x - \nu(u_{xx}+u_{yy}+u_{zz}) &= 0,\\
v_t + u v_x + v v_y + w v_z + p_y - \nu(v_{xx}+v_{yy}+v_{zz}) &= 0,\\
w_t + u w_x + v w_y + w w_z + p_z - \nu(w_{xx}+w_{yy}+w_{zz}) &= 0,\\
u_x+v_y+w_z &= 0.
\end{align*}
The default setting uses $\mathrm{Re}=1$, hence $\nu=1$, with parameters $a=d=1$. The analytical solution is
\begin{align*}
u^\star &= -a\left[
e^{a x}\sin(a y+d z)+e^{a z}\cos(a x+d y)
\right]e^{-d^2 t},\\
v^\star &= -a\left[
e^{a y}\sin(a z+d x)+e^{a x}\cos(a y+d z)
\right]e^{-d^2 t},\\
w^\star &= -a\left[
e^{a z}\sin(a x+d y)+e^{a y}\cos(a z+d x)
\right]e^{-d^2 t},
\end{align*}
and
\begin{align*}
p^\star
=
-\frac{1}{2}a^2
\Big[
&e^{2ax}+e^{2ay}+e^{2az}
+2\sin(ax+dy)\cos(az+dx)e^{a(y+z)} \\
&+2\sin(ay+dz)\cos(ax+dy)e^{a(z+x)} \\
&+2\sin(az+dx)\cos(ay+dz)e^{a(x+y)}
\Big]e^{-2d^2t}.
\end{align*}
Initial and boundary conditions are sampled from this analytical solution, and the same analytical solution is used for validation and testing.

\paragraph{Poisson-5D.}
The Poisson-5D benchmark from PINNacle solves a scalar Poisson equation on the five-dimensional hypercube $\Omega=[0,1]^5$:
\begin{equation*}
\Delta u(x) + f(x)=0,
\end{equation*}
where
\begin{equation*}
f(x)=\frac{\pi^2}{4}\sum_{j=1}^{5}\sin\left(\frac{\pi}{2}x_j\right).
\end{equation*}
The exact solution is
\begin{equation*}
u^\star(x)=\sum_{j=1}^{5}\sin\left(\frac{\pi}{2}x_j\right).
\end{equation*}
Dirichlet boundary conditions are given by the exact solution, $u=u^\star$ on $\partial\Omega$. The reference solution is generated analytically from $u^\star$.

\newpage
\subsection{Network Architectures}

All PINN models use fully connected multilayer perceptrons (MLPs). Each model
maps the spatio-temporal coordinates of a collocation point to a shared hidden
representation and then to the PDE-specific output channels. Concretely, for a
benchmark with input coordinate \(x\in\mathbb R^{d_{\mathrm{in}}}\), the network
has the form
\[
\hat u_\theta(x)
=
W_{L+1}\phi\!\left(
W_L\phi\!\left(\cdots \phi(W_1 x+b_1)\cdots\right)+b_L
\right)+b_{L+1},
\]
where \(L\) is the number of hidden blocks, the hidden width is shared within a
benchmark, and \(\phi\) is the activation function listed in
Table~\ref{tab:network_architectures}. We use one shared trunk per benchmark rather than separate subnetworks for
different loss terms. The residual, boundary, and initial-condition losses are
therefore coupled through the same parameter vector $\theta$.

The input coordinates are the physical coordinates used by the corresponding
benchmark runner. Time-dependent scalar equations use coordinates such as
\((x,t)\), steady two-dimensional flow problems use \((x,y)\), and
high-dimensional Poisson uses its spatial coordinates directly. The output
dimension is determined by the PDE state: scalar equations output one field,
complex or multi-species systems output multiple channels, and incompressible
flow benchmarks output velocity components together with pressure when pressure
appears in the residual. Spatial and temporal derivatives in the PDE residuals
are computed by automatic differentiation through this MLP, so the same network
must support all derivative orders required by the benchmark equation.

Architectures are fixed across all baseline gradient surgery methods and their
+GUA variants. We do not retune network depth, width, activation, or output
parameterization for GUA, so the comparison isolates the effect of GUA rather
than architectural differences. All models use Xavier initialization, and for
each random seed, the base method and its +GUA counterpart use the same
initialization.

\begin{table*}[!htbp]
\centering
\caption{PINN network architectures. Shared architecture fields are repeated to
make the per-benchmark settings explicit.}
\label{tab:network_architectures}
\normalsize
\setlength{\tabcolsep}{3pt}
\renewcommand{\arraystretch}{1.1}
\begin{tabular}{cccccc}
\toprule
\textbf{Benchmark} & \textbf{Input coordinates} & \textbf{Output dim.} &
\textbf{Hidden width} & \textbf{Hidden layers} & \textbf{Activation} \\
\midrule
Burgers  & $(x,t)$ & 1 & 50 & 5 & Tanh \\
Schr\"odinger  & $(x,t)$ & 2 & 50 & 5 & Tanh \\
Kovasznay  & $(x,y)$ & 3 & 50 & 5 & Tanh \\
Poisson-5D  & $(x_1,\ldots,x_5)$ & 1 & 50 & 5 & Tanh \\
Heat-MS  & $(x,y,t)$ & 1 & 50 & 5 & Tanh \\
Beltrami  & $(x,y,z,t)$ & 4 & 50 & 5 & Tanh \\
\bottomrule
\end{tabular}
\end{table*}

\newpage

\subsection{Training Settings}
\label{app:training_settings}

All PINN experiments use Xavier initialization and full-batch training over the
sampled collocation, boundary, and initial-condition points. Unless otherwise
stated, training points are resampled every epoch using Latin-hypercube
sampling. Adam is used as the default optimizer with $\beta_1=0.9$,
$\beta_2=0.999$, and $\epsilon=10^{-8}$. We use no manual scalar reweighting,
with all loss weights set to $\lambda_i=1$ before method-specific gradient
manipulation. Validation is performed every 100 epochs. For each run, we select
the checkpoint with the best validation performance and use it for the final
test evaluation. Each reported PINN result is presented as the mean and
standard deviation over five independent runs using seeds $\{0,1,2,3,4\}$.
All experiments use a 100-epoch linear warmup. In the default setting, the
warmup is followed by cosine decay. Let $w=100$ denote the number of warmup
epochs. At scheduler index $t$, the learning-rate schedule is
\[
\eta_t
=
\begin{cases}
\displaystyle \eta_0\frac{t}{w},
& 0\leq t<w,\\[2mm]
\displaystyle \eta_{\min}
+\frac{1}{2}(\eta_0-\eta_{\min})
\left[
1+\cos\left(\frac{\pi(t-w)}{T-w}\right)
\right],
& w\leq t<T,
\end{cases}
\]
where the first training epoch uses scheduler index $t=0$, $T$ denotes the
configured number of training epochs, and the last applied update uses
$t=T-1$. The default setting uses
$\eta_0=10^{-3}$ and
$\eta_{\min}=10^{-4}$.
The only learning-rate exception is Heat-MS under the 3-loss setting. With
$\eta_0=10^{-3}$, the baseline ConFIG configuration exhibits numerical
instability. To avoid confounding the comparison with this baseline
instability, we set both $\eta_0$ and $\eta_{\min}$ to $10^{-4}$ for all
methods in this setting. After the same 100-epoch linear warmup, the learning
rate remains constant at $10^{-4}$.
Tables~\ref{tab:global_training_settings} and
\ref{tab:training_settings} summarize the training settings and configurations, respectively.

\begin{table}[!htbp]
\centering
\caption{Global PINN training settings.}
\label{tab:global_training_settings}
\normalsize
\setlength{\tabcolsep}{6pt}
\renewcommand{\arraystretch}{1.1}
\begin{tabular}{cc}
\toprule
\textbf{Item} & \textbf{Setting} \\
\midrule
Initialization & Xavier \\
Sampling & Latin-hypercube sampling, resampled every epoch \\
Optimizer & Adam ($\beta_1=0.9$, $\beta_2=0.999$, $\epsilon=10^{-8}$) \\
Default learning rate & 100-epoch warmup, then cosine decay from $10^{-3}$ to $10^{-4}$ \\
Heat-MS/3-loss learning rate & 100-epoch warmup to $10^{-4}$, then constant at $10^{-4}$ \\
Batching & Full batch \\
Loss weights & $\lambda_i=1$ \\
Validation interval & Every 100 epochs \\
Runs & 5 seeds: $\{0,1,2,3,4\}$ \\
\bottomrule
\end{tabular}
\end{table}

\begin{table}[h]
\centering
\caption{PINN training configurations used in the experiments. $N_f$,
$N_b$, and $N_i$ denote the numbers of interior collocation, boundary, and
initial-condition points, respectively.}
\label{tab:training_settings}
\normalsize
\setlength{\tabcolsep}{9pt}
\renewcommand{\arraystretch}{1.1}
\begin{tabular}{ccccc}
\toprule
\textbf{Benchmark} & \textbf{Epochs} & \textbf{$N_f$} & \textbf{$N_b$} &
\textbf{$N_i$} \\
\midrule
Burgers & 30,000 & 10,000 & 250 & 250 \\
Schr\"odinger & 100,000 & 20,000 & 500 & 500 \\
Kovasznay  & 100,000 & 20,000 & 1,000 & -- \\
Poisson-5D & 100,000 & 20,000 & 5,000 & -- \\
Heat-MS & 100,000 & 20,000 & 2,000 & 2,000 \\
Beltrami  & 100,000 & 25,000 & 5,000 & 5,000 \\
\bottomrule
\end{tabular}
\end{table}

\newpage
\subsection{Optimizer Configurations}
\label{app:optimizer_descriptions}

This section summarizes the optimizer configurations used in the Q1
optimizer-induced GUM experiments. We fix ConFIG as the base gradient surgery
method. At each training step, ConFIG constructs a pre-optimizer direction
$a_t$, which is then passed to the corresponding optimizer to produce an
optimizer proposal $u_t$. No GUA is applied in these experiments, so the
proposal is used directly for the parameter update:
\[
    \theta_{t+1} = \theta_t - \eta_t u_t .
\]

The purpose of this section is therefore to specify how each optimizer
transforms the constructed direction $a_t$ into the proposal $u_t$,
\[
    a_t \mapsto u_t.
\]

For notational simplicity, division, square root, clipping, and squaring are
applied element-wise unless otherwise stated. We use $\epsilon>0$ for numerical
stability, $\odot$ for element-wise multiplication, and $\lambda$ for the
weight-decay coefficient when applicable. For adaptive optimizers, $m_t$ and
$v_t$ denote the first- and second-moment states, respectively. Barred
quantities denote the tentative states after processing $a_t$ at step $t$.
These states produce the proposal $u_t$ under the time-indexing convention
introduced in Eq.~(\ref{eq:optimizer_transformation}).
\begin{table}[!htbp]
\centering
\caption{Optimizer update maps used in the optimizer-induced GUM experiments.}
\label{tab:optimizer_descriptions}
\scriptsize
\setlength{\tabcolsep}{4pt}
\renewcommand{\arraystretch}{1.4}
\begin{tabular}{@{}ccl@{}}
\toprule
\textbf{Family} & \textbf{Optimizer} & \textbf{Effective update map} \\
\midrule
First-order
& SGD
& \(u_t = a_t\) \\
\addlinespace[1.5pt]

\rowcolor{gray!10} First-order
& M-SGD
& \(\bar m_{t+1} = \beta m_t + a_t,\quad u_t = \bar m_{t+1}\) \\
\addlinespace[1.5pt]

First-order
& RMSProp
& \(\bar v_{t+1} = \rho v_t + (1-\rho)a_t^2,\quad
   u_t = a_t / (\sqrt{\bar v_{t+1}}+\epsilon)\) \\
\addlinespace[1.5pt]

\rowcolor{gray!10} First-order
& Adam
& $\begin{aligned}
  &\bar m_{t+1}=\beta_1m_t+(1-\beta_1)a_t,\quad
    \bar v_{t+1}=\beta_2v_t+(1-\beta_2)a_t^2 \\
  &\hat{\bar m}_{t+1}=\bar m_{t+1}/(1-\beta_1^{t+1}),\quad
    \hat{\bar v}_{t+1}=\bar v_{t+1}/(1-\beta_2^{t+1}),\quad
    u_t=\hat{\bar m}_{t+1}/(\sqrt{\hat{\bar v}_{t+1}}+\epsilon)
\end{aligned}$ \\
\addlinespace[1.5pt]

First-order
& AdamW
& \(u_t=\hat{\bar m}_{t+1}/(\sqrt{\hat{\bar v}_{t+1}}+\epsilon)+\lambda\theta_t\) \\
\addlinespace[1.5pt]

\midrule
\rowcolor{gray!10} Second-order
& SophiaG
& $\begin{aligned}
  &\bar m_{t+1}=\beta_1m_t+(1-\beta_1)a_t,\quad
    \bar h_{t+1}=\beta_2h_t+(1-\beta_2)\hat h_t \\
  &u_t=\operatorname{clip}\!\left(\bar m_{t+1}/(\gamma\bar h_{t+1}+\epsilon),\,\rho\right)
\end{aligned}$ \\
\addlinespace[1.5pt]

Second-order
& AdaHessian
& $\begin{aligned}
  &\bar m_{t+1}=\beta_1m_t+(1-\beta_1)a_t,\quad
    \bar h_{t+1}=\beta_2h_t+(1-\beta_2)\hat h_t \\
  &u_t=\bar m_{t+1}/(\bar h_{t+1}^{1/2}+\epsilon)
\end{aligned}$ \\
\addlinespace[1.5pt]

\rowcolor{gray!10} Second-order
& SOAP
& $\begin{aligned}
  &\tilde a_t = Q_t^{\top}a_t,\quad
    \tilde{\bar m}_{t+1}=\beta_1\tilde m_t+(1-\beta_1)\tilde a_t \\
  &\tilde{\bar v}_{t+1}=\beta_2\tilde v_t+(1-\beta_2)\tilde a_t^2,\quad
    u_t=Q_t\!\left(\tilde{\bar m}_{t+1}/(\sqrt{\tilde{\bar v}_{t+1}}+\epsilon)\right)
\end{aligned}$ \\
\bottomrule
\end{tabular}
\end{table}

\newpage
\subsection{Baseline Settings}
\label{app:baseline_settings}

We compare GUA with several representative gradient surgery methods. All gradient surgery methods first construct a direction $a_t$ from the loss-specific gradients $\{g_{i,t}\}_{i=1}^{m}$, and then feed this direction into the same base optimizer. This allows us to evaluate whether the direction constructed by gradient
surgery before optimizer transformation is preserved after the optimizer
transformation.

\paragraph{PCGrad.}
PCGrad~\cite{yu2020gradient} mitigates gradient conflict by modifying loss-specific gradients pairwise. When two loss-specific gradients have a negative inner product, PCGrad projects one gradient onto the normal plane of the other to remove the conflicting component. The modified gradients are then aggregated to form the constructed direction $a_t$. PCGrad is a representative projection-based gradient surgery method.

\paragraph{CAGrad.}
CAGrad~\cite{liu2021conflict} constructs a direction that balances conflict reduction with optimization of the average loss. Instead of only removing pairwise conflicts, CAGrad searches for a direction near the average gradient that improves the worst local task descent. This makes it a conflict-averse aggregation method that tries to reduce task interference while remaining aligned with the average-loss objective.

\paragraph{IMTL-G.}
IMTL-G~\cite{liu2021towards} aims to obtain an impartial direction across tasks. It computes aggregation weights such that the resulting direction has balanced projections onto different loss-specific gradients. In this way, IMTL-G avoids favoring a subset of objectives and provides a balanced constructed direction for multi-loss training.

\paragraph{A-MTL.}
A-MTL~\cite{senushkin2023independent} addresses gradient conflict and gradient dominance by aligning components of the gradient matrix. It treats the set of loss-specific gradients as a linear system and improves optimization stability by reducing ill-conditioning and dominance among loss-specific gradients. We include A-MTL as a representative gradient alignment method beyond pairwise projection or cone-based constraints.

\paragraph{UPGrad.}
UPGrad~\cite{quinton2024jacobian} is based on dual-cone projection in multi-objective optimization. It projects loss-specific gradients onto the dual cone induced by the set of objective gradients and aggregates the projected directions. The resulting update is designed to avoid locally increasing any objective under the current first-order geometry.

\paragraph{ConFIG.}
ConFIG~\cite{liu2025config} is a conflict-free direction construction method designed for multi-loss training in PINNs. Given loss-specific gradients, ConFIG constructs a direction whose inner product with each loss-specific gradient is non-negative. Therefore, its constructed direction satisfies $a_t\in\mathcal C_t$ under the current gradient geometry.

\newpage
\subsection{Evaluation Metrics}
\label{app:evaluation_metrics}

For PINN benchmarks, we report the relative $L_2$ error against the reference
solution and the four-stage conflict-rate sequence
$R_g\!\rightarrow R_a\!\rightarrow R_u\!\rightarrow R_p$. Here $R_g$
measures raw loss-specific gradient conflict, $R_a$ the constructed direction,
$R_u$ the optimizer proposal, and $R_p$ the aligned update. Since a base
method does not construct $p_t$, $R_p$ is not applicable to baseline runs.

Let $\{x_j\}_{j=1}^{N_{\mathrm{test}}}$ denote the test points, let
$\hat u_\theta$ be the PINN prediction, and let $u^\star$ be the reference
solution. The relative $L_2$ error is
\begin{equation*}
\mathrm{Rel}\text{-}L_2
=
\frac{
\left(\sum_{j=1}^{N_{\mathrm{test}}}
\|\hat u_\theta(x_j)-u^\star(x_j)\|_2^2\right)^{1/2}
}{
\left(\sum_{j=1}^{N_{\mathrm{test}}}
\|u^\star(x_j)\|_2^2\right)^{1/2}
}.
\end{equation*}
All methods and variants use the same benchmark-specific evaluation points.
In particular, Poisson-5D uses the deterministic PINNacle-compatible uniform
grid produced by eight interior coordinates per dimension, giving
\(8^5=32{,}768\) evaluation points.

At training step $t$, let $\{g_{i,t}\}_{i=1}^{m}$ be the loss-specific
gradients, $a_t$ the constructed direction produced by the base gradient
surgery method, and $u_t$ the optimizer proposal produced by the base optimizer.
The current conflict-free cone is
\begin{equation*}
\mathcal C_t
=
\left\{
d:\ \langle g_{i,t},d\rangle\ge 0,\quad i=1,\ldots,m
\right\}.
\end{equation*}
These quantities use the same ideal cone-membership definitions as in the main
text. In the actual implementation, a conflict is detected when the normalized inner
product falls below $-10^{-6}$. We use a small negative threshold rather than
zero to avoid treating tiny negative values caused by floating-point and
numerical errors near the conflict boundary as genuine conflicts. The
denominator includes $\epsilon$ for numerical stability, and only diagnosed
steps with a positive learning rate are included.
For a generic update direction $d_t$, the empirical conflict indicator is
\begin{equation*}
\mathbb I_t(d_t)
=
\mathbb I\!\left[
\exists i\in\{1,\ldots,m\}:
\frac{\langle g_{i,t},d_t\rangle}
{\|g_{i,t}\|_2\|d_t\|_2+\epsilon}<-10^{-6}
\right].
\end{equation*}
The raw loss-specific gradient conflict rate is
\begin{equation*}
R_g
=
\frac{1}{T}\sum_{t=1}^{T}
\mathbb I\!\left[\exists i<j:\
\frac{\langle g_{i,t},g_{j,t}\rangle}
{\|g_{i,t}\|_2\|g_{j,t}\|_2+\epsilon}<-10^{-6}\right].
\end{equation*}
Here \(T\) counts diagnosed steps with a positive learning rate. The zero-rate warmup step is omitted because it produces no applied update.
We report the constructed-direction $a_t$, optimizer-proposal $u_t$, and aligned-update $p_t$
conflict rates as
\begin{equation*}
R_a
=
\frac{1}{T}\sum_{t=1}^{T}\mathbb I_t(a_t),
\qquad
R_u
=
\frac{1}{T}\sum_{t=1}^{T}\mathbb I_t(u_t),
\qquad
R_p
=
\frac{1}{T}\sum_{t=1}^{T}\mathbb I_t(p_t).
\end{equation*}
The complete diagnostic sequence is therefore
$R_g\!\rightarrow R_a\!\rightarrow R_u\!\rightarrow R_p$.
Baseline tables report the first three stages and use ``--'' for $R_p$.

\newpage
\subsection{GUA Implementation Details}
\label{app:implementation_details}
\label{app:gua_algorithm}

This section provides the implementation details corresponding to Algorithm~\ref{alg:gua}.

\begin{algorithm}[t]
\caption{Gradient--Update Alignment (GUA)}
\label{alg:gua}
\begin{algorithmic}[1]
\Require Current parameters $\theta_t$, optimizer state $s_t$, losses $\{\mathcal L_i\}_{i=1}^{m}$, base gradient surgery method, base optimizer $\mathcal O_t$, learning rate $\eta_t$
\Ensure Updated parameters $\theta_{t+1}$ and optimizer state $s_{t+1}$
\State Compute loss-specific gradients $\{g_{i,t}\}_{i=1}^{m}$
\State Construct the pre-optimizer direction $a_t$ using the base gradient surgery method
\State Obtain the optimizer proposal: $(u_t,\bar s_{t+1}) \gets \mathcal O_t(a_t,s_t)$
\State \textbf{Direction  alignment:} project $u_t$ onto the conflict-free cone:
$p_t \gets \arg\min_{d\in\mathcal C_t}\frac{1}{2}\|d-u_t\|_{M_t}^{2}$
\State Apply the aligned update: $\theta_{t+1}\gets\theta_t-\eta_t p_t$
\State \textbf{State alignment:} if the projection modifies the proposal, compute \(\tilde s_{t+1}\gets\Phi_t(p_t,u_t,\bar s_{t+1})\) and set $s_{t+1}$ by Eq.~(\ref{eq:state_alignment_generic}); otherwise set $s_{t+1}\gets\bar s_{t+1}$
\State \Return $(\theta_{t+1},s_{t+1})$
\end{algorithmic}
\end{algorithm}

\subsubsection{Projection Solver}
\label{app:projection_solver}
Let $G_t=[g_{1,t},\ldots,g_{m,t}]$ collect the loss-specific gradients as
flattened vectors, and let
\(\mathcal C_t=\{d:G_t^\top d\ge0\}\) denote the current conflict-free
cone. Given the optimizer proposal \(u_t\), GUA computes a reference
update \(p_t\in\mathcal C_t\) that stays close to \(u_t\) under a chosen
projection metric:
\[
p_t
=
\arg\min_d \frac{1}{2}\|d-u_t\|_{M_t}^2
\quad
\mathrm{s.t.}\quad
G_t^\top d \ge 0 .
\]
The projection is performed in update space: the cone $\mathcal C_t$ is induced by
the current loss-specific gradients, and GUA projects the optimizer proposal
after optimizer transformation rather than the parameters themselves. The
Euclidean metric sets $M_t=I$, yielding the standard least-change
projection in the original parameter space:
\[
p_t^{\mathrm{Euc}}
=
\arg\min_d \frac{1}{2}\|d-u_t\|_2^2
\quad
\mathrm{s.t.}\quad
G_t^\top d \ge 0 .
\]
The Adam metric instead measures the projection distance in the geometry
induced by Adam's second-moment state. Let \(\hat{\bar v}_{t+1}\) be the
bias-corrected tentative second-moment state associated with $u_t$. We use a
diagonal metric
\[
M_t=\operatorname{diag}(\sqrt{\hat{\bar v}_{t+1}}+\epsilon),
\qquad
\|x\|_{M_t}^2=x^\top M_t x .
\]
Thus,
\[
p_t^{\mathrm{Adam}}
=
\arg\min_d \frac{1}{2}\|d-u_t\|_{M_t}^2
\quad
\mathrm{s.t.}\quad
G_t^\top d \ge 0 .
\]
In practice, \(M_t\) can be chosen at different levels of structure. The
identity metric is optimizer-agnostic and gives the smallest implementation
surface. A diagonal metric, such as the Adam metric above, is optimizer-guided
while preserving the \(O(m^2p)\) Gram-matrix cost and avoiding any \(p\times p\)
storage. Block-diagonal or full preconditioner metrics could be used when an
optimizer exposes such structure, but they require applying \(M_t^{-1}\) to
loss-gradient vectors and may increase both memory and implementation cost.
For this reason, our experiments restrict \(M_t\) to the two practically useful
choices \(I\) and diagonal Adam metric.

Official GUA experiments use the Adam metric by default for both two-loss and
three-loss decompositions. The projection is performed on the flattened
full-network parameter vector, not separately per layer, because the
conflict-free cone is defined by global inner products with loss gradients;
layer-wise projection would instead impose separate cones on individual layers
and change the intended update-level geometry. Before solving, we
normalize each nonzero constraint gradient and the optimizer proposal for
numerical stability, and restore the proposal scale after projection. This
positive rescaling leaves the conflict-free cone unchanged, so the constraints
remain induced by the same loss-specific gradients computed by the base
gradient surgery method.

The projection can be solved entirely in the \(m\)-dimensional dual, where \(m\)
is the number of losses. In our PINN experiments, \(m\in\{2,3\}\), so we solve
this dual problem exactly by enumerating active sets. Define
\[
b=G_t^\top u_t,\qquad K=G_t^\top M_t^{-1}G_t .
\]
If \(b\ge0\), the optimizer proposal is already conflict-free and we return
\(p_t=u_t\). Otherwise, we enumerate each active constraint set
\(S\subseteq\{1,\ldots,m\}\). For each active set, we solve the small
KKT system
\[
K_{SS}\lambda_S=-b_S,
\]
set inactive multipliers to zero, and check the KKT feasibility conditions
\[
\lambda_S\ge0,\qquad b+K\lambda\ge0 .
\]
Each feasible active set yields
\[
p_t(\lambda)=u_t+M_t^{-1}G_t\lambda .
\]
Among feasible active sets, we choose the one with the smallest projection
objective. This enumeration is used only for the small-\(m\) PINN setting. In
the Euclidean case \(M_t=I\), this reduces to
\(K=G_t^\top G_t\) and \(p_t=u_t+G_t\lambda\). In the Adam-metric case,
\(M_t\) is the positive diagonal metric defined above. Since \(m\in\{2,3\}\)
in our PINN experiments, enumerating all active sets is inexpensive and gives
the exact projection for the chosen metric.

The projection does not form or store any \(p\times p\) matrix. Its
parameter-dimensional cost is dominated by constructing the Gram matrix
\(K=G_t^\top M_t^{-1}G_t\), which costs \(O(m^2p)\) for diagonal \(M_t\) after
the loss-specific gradients have been computed. Active-set enumeration costs
\(O(2^m m^3)\) in the worst case for solving the small dual systems, so it is
appropriate only when \(m\) is small. For CelebA settings with \(m>8\), we
instead solve the same projection as the
low-dimensional nonnegative dual quadratic program
\[
    \min_{\lambda\ge0}
    \frac12\lambda^\top K\lambda + b^\top\lambda ,
\]
using an active-set solver over the \(m\)-dimensional dual variable. This avoids
enumerating \(2^m\) subsets while keeping the large parameter dimension outside
the QP. The large parameter dimension enters only through gradient inner
products, and the remaining optimization is over the number of losses.

\newpage
\subsubsection{Optimizer-State Alignment}
\label{app:optimizer_state_alignment}

The state-alignment rule is defined in Eq.~(\ref{eq:state_alignment_generic}). Here we detail its Adam instantiation, as Adam is used by the gradient-surgery baselines in their standard settings and admits a simple state reconstruction through its first- and second-moment statistics. 

Let $\bar m_{t+1}$ and $\bar v_{t+1}$ denote the tentative Adam first- and
second-moment states associated with the optimizer proposal $u_t$. Their
bias-corrected versions are
\[
\hat{\bar m}_{t+1}
=
\frac{\bar m_{t+1}}{1-\beta_1^{t+1}},
\qquad
\hat{\bar v}_{t+1}
=
\frac{\bar v_{t+1}}{1-\beta_2^{t+1}}.
\]

Given the applied update $p_t$, we construct the pseudo-gradient
\[
g^{\mathrm{corr}}_{t+1}
=
(\sqrt{\hat{\bar v}_{t+1}}+\epsilon)
\odot p_t,
\]
which would produce $p_t$ under the current Adam denominator. The target Adam
states are
\[
\tilde m_{t+1}
=
(1-\beta_1^{t+1})g^{\mathrm{corr}}_{t+1},
\qquad
\tilde v_{t+1}
=
(1-\beta_2^{t+1})
(g^{\mathrm{corr}}_{t+1})^{\odot 2}.
\]

Finally, GUA softly aligns the stored Adam moments:
\[
m_{t+1}
=
(1-\rho_m)\bar m_{t+1}
+
\rho_m\tilde m_{t+1},
\qquad
v_{t+1}
=
(1-\rho_v)\bar v_{t+1}
+
\rho_v\tilde v_{t+1}.
\]

Here, $\rho_m$ and $\rho_v$ control the alignment strengths of the first and
second moments, respectively. The values used in our experiments are reported
in Table~\ref{tab:state_alignment_rho} and are selected based on validation
performance.

\begin{table}[h]
\centering
\caption{Coefficients used for optimizer-state alignment.}
\label{tab:state_alignment_rho}
\normalsize
\begin{tabular}{ccc}
\toprule
Equation & \(\rho_m\) & \(\rho_v\) \\
\midrule
Burgers & 0.10 & 0.03 \\
Schr\"odinger & 0.10 & 0.03 \\
Kovasznay & 0.70 & 0.20 \\
Heat-MS & 0.70 & 0.20 \\
Poisson-5D & 0.50 & 0.15 \\
Beltrami & 0.70 & 0.20 \\
\bottomrule
\end{tabular}
\end{table}
Notably, the state-alignment principle is not restricted to Adam. It can be extended to other stateful optimizers when target internal states can be reconstructed from the applied update according to the corresponding update rules. The corresponding implementations for different optimizers and their results are reported in Appendix~\ref{app:gua_cross_optimizer}.

\newpage
\section{Additional Results, Ablations, and Diagnostics}
\label{app:additional_gum_diagnostics}

\subsection{Projection Metric Ablation}
\label{app:projection_metric_diagnostic}

As defined in Eq.~(\ref{eq:gua_projection}), GUA projects the optimizer
proposal $u_t$ onto the conflict-free cone $\mathcal C_t$ under the metric
$M_t$.
We compare two choices of $M_t$: the Euclidean metric $M_t=I$ and the
Adam metric defined by Adam's diagonal second-moment statistics. The Euclidean
metric treats all parameter coordinates uniformly, whereas the Adam metric
measures the projection distance in the optimizer-induced coordinate geometry.
Figure~\ref{fig:projection_metric_ablation} compares their accuracy and numerical
stability across PINN settings. A hatched bar denotes partial numerical
instability, while a cross denotes complete numerical failure.

Overall, the Adam metric provides accuracy comparable to or better than the
Euclidean metric across the evaluated settings. The difference is modest for
some equations, such as Schr\"odinger and Beltrami flow, but becomes more
pronounced for Burgers, Kovasznay flow, and Poisson-5D. The clearest distinction
appears on Heat-MS: under the Euclidean metric, the two-loss setting exhibits
partial numerical instability and the three-loss setting becomes entirely
non-finite, whereas the Adam metric remains numerically stable in both cases.

These results indicate that the projection metric affects more than the
feasibility of the corrected update. Both choices enforce the same
conflict-free constraints, but they select different feasible updates by
measuring the deviation from $u_t$ in different geometries. The Euclidean
metric ignores the coordinate-wise scaling already induced by Adam, whereas
the Adam metric accounts for this anisotropy through the optimizer's
second-moment statistics. The improved stability observed on Heat-MS is
therefore consistent with preserving the optimizer-induced geometry during
projection. Since the Adam metric is at least competitive in accuracy and
provides better numerical robustness in the difficult settings, we use it as
the default projection metric in the main experiments.

\begin{figure*}[!htbp]
\centering
\includegraphics[width=0.84\textwidth]{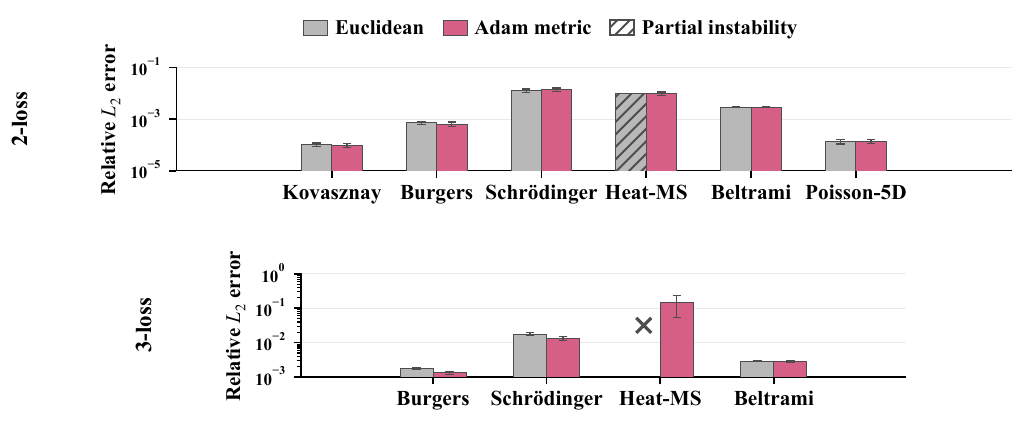}
\caption{Euclidean and Adam projection metrics across PINN equations. Hatching marks partial numerical instability, and a \(\times\) marks complete numerical failure.}
\label{fig:projection_metric_ablation}
\end{figure*}

\newpage
\subsection{State-Alignment Sensitivity}
\label{app:rho_sensitivity}

State alignment introduces two coefficients, $\rho_m$ and $\rho_v$, controlling
how strongly the Adam moments are moved toward the target state. We examine
whether GUA is sensitive to these coefficients using seven predefined pairs,
\begin{equation*}
\begin{aligned}
(\rho_m,\rho_v)\in\{&
(0.003,0.001),(0.01,0.003),(0.03,0.01),(0.1,0.03),\\
& (0.3,0.1),(0.5,0.15),(0.7,0.2)\},
\end{aligned}
\end{equation*}
across all PINN settings, six gradient surgery methods, and five seeds.

The coefficient pair used in the main experiments is selected using validation
performance, as described in Appendix~\ref{app:optimizer_state_alignment}, and
shared across loss decompositions and gradient surgery methods for each equation.
The test set is not used for selection.
Figure~\ref{fig:rho_sensitivity} reports sensitivity over the full coefficient
grid using the geometric mean of paired relative-error ratios to the corresponding
baseline, where values below one indicate improvement. Validation-selected
coefficients are highlighted, and the single numerical failure is retained.
Overall, GUA is broadly robust to the state-alignment coefficients, with relative-error
ratios below one across most equations, methods, and coefficient pairs. Burgers,
Beltrami flow, and Poisson-5D are particularly robust. Sensitivity varies across
equations: Schr\"odinger and Burgers favor moderate alignment, Heat-MS tends to
benefit from stronger alignment, while Kovasznay flow and Poisson-5D show greater
method-dependent variation.

\begin{figure}[!h]
\centering
\includegraphics[width=\textwidth]{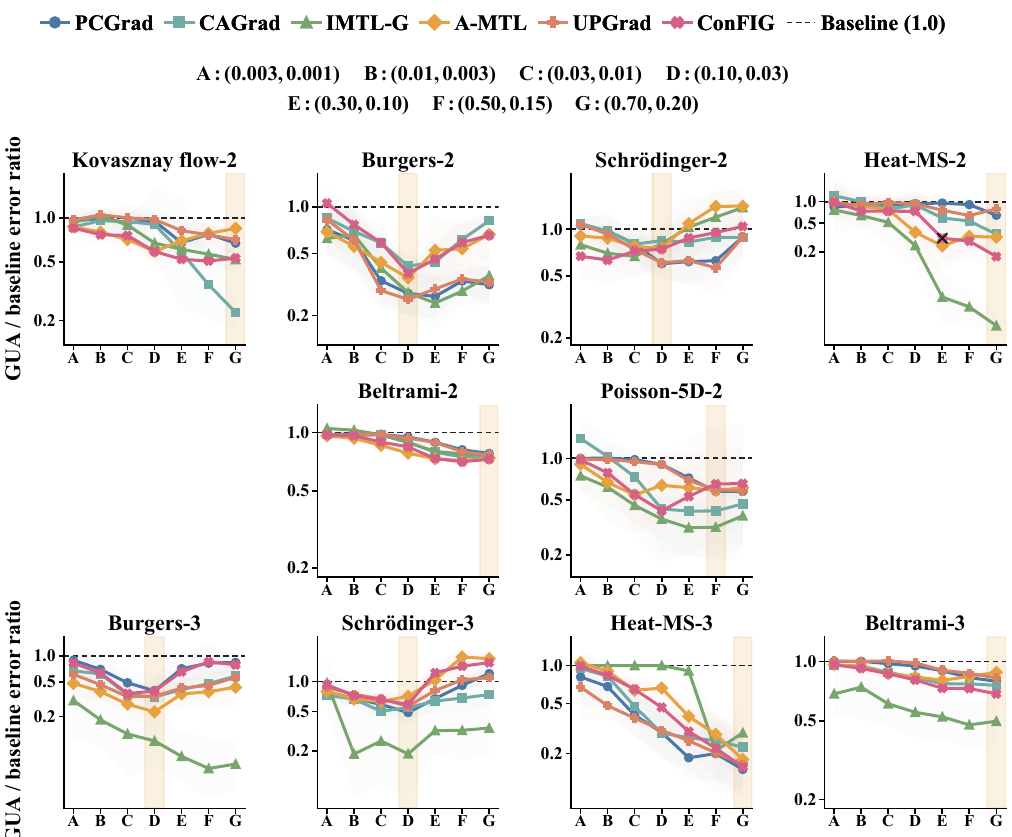}
\caption{State-alignment sensitivity across PINN equations and gradient surgery methods. Shaded bands show standard deviation, the dashed line marks baseline parity, orange bands identify the selected coefficients, and a black cross marks numerical failure.}
\label{fig:rho_sensitivity}
\end{figure}

\newpage
\subsection{Optimizer-Induced GUM Results on Kovasznay, Beltrami, and Poisson-5D}
\label{app:additional_optimizer_gum}

Owing to space limitations, the main text presents Q1 optimizer-induced GUM results for a subset of the PINN benchmarks, with additional results provided in the appendix.
Table~\ref{tab:optimizer_sources_remaining}
reports the corresponding results for Kovasznay, Beltrami, and Poisson-5D under the same
experimental protocol. Consistent with the main-text results, ConFIG yields
$R_a=0$, while momentum, adaptive, and curvature-based optimizers can produce
nonzero post-optimizer conflict rates $R_u$. Plain SGD preserves $R_u=0$, as
expected from $u_t=a_t$.

\begin{table*}[h]
\centering
\caption{Optimizer-induced GUM across the remaining PINN settings. The sequence \(R_g \rightarrow R_a \rightarrow R_u\) tracks conflict rates ($\downarrow$) from raw loss-specific gradients through gradient surgery to optimizer transformation.}
\label{tab:optimizer_sources_remaining}
\scriptsize
\setlength{\tabcolsep}{2pt}
\renewcommand{\arraystretch}{1.3}
\begin{tabular}{c *{3}{w{c}{16pt}} @{\hspace{6pt}} *{3}{w{c}{16pt}} @{\hspace{6pt}} *{3}{w{c}{16pt}} @{\hspace{6pt}} *{3}{w{c}{16pt}}}
\toprule
\multicolumn{1}{c}{\multirow{3}{*}[-0.5ex]{\textbf{Optimizer}}}
& \multicolumn{3}{c}{\textbf{Kovasznay}}
& \multicolumn{6}{c}{\textbf{Beltrami}}
& \multicolumn{3}{c}{\textbf{Poisson-5D}} \\
\cmidrule(lr){2-4} \cmidrule(lr){5-10} \cmidrule(lr){11-13}
& \multicolumn{3}{c}{\textbf{2-loss}} & \multicolumn{3}{c}{\textbf{2-loss}} & \multicolumn{3}{c}{\textbf{3-loss}} & \multicolumn{3}{c}{\textbf{2-loss}} \\
\cmidrule(lr){2-4} \cmidrule(lr){5-7} \cmidrule(lr){8-10} \cmidrule(lr){11-13}
& $R_g$ & $R_a$ & $R_u$ & $R_g$ & $R_a$ & $R_u$ & $R_g$ & $R_a$ & $R_u$ & $R_g$ & $R_a$ & $R_u$ \\
\midrule
\rowcolor{gray!22}\multicolumn{13}{c}{\textbf{First-order optimizers}} \\
SGD & 100.0 & 0.0 & 0.0 & 100.0 & 0.0 & 0.0 & 100.0 & 0.0 & 0.0 & 87.9 & 0.0 & 0.0 \\
M-SGD & 99.9 & 0.0 & 53.3 & 79.0 & 0.0 & 22.9 & 90.3 & 0.0 & 43.2 & 64.9 & 0.0 & 41.9 \\
RMSProp & 10.0 & 0.0 & 1.7 & 3.3 & 0.0 & 0.0 & 32.5 & 0.0 & 0.9 & 81.9 & 0.0 & 0.2 \\
Adam & 5.5 & 0.0 & 6.4 & 13.1 & 0.0 & 11.6 & 27.5 & 0.0 & 16.8 & 52.6 & 0.0 & 30.5 \\
AdamW & 8.9 & 0.0 & 8.4 & 14.6 & 0.0 & 10.7 & 29.0 & 0.0 & 17.4 & 52.9 & 0.0 & 32.6 \\
\midrule
\rowcolor{gray!22}\multicolumn{13}{c}{\textbf{Second-order optimizers}} \\
SophiaG & 55.6 & 0.0 & 42.9 & 54.1 & 0.0 & 20.2 & 63.0 & 0.0 & 29.4 & 27.2 & 0.0 & 13.0 \\
AdaHessian & 99.9 & 0.0 & 36.8 & 99.9 & 0.0 & 53.2 & 99.9 & 0.0 & 57.7 & 98.8 & 0.0 & 8.5 \\
SOAP & 3.3 & 0.0 & 5.8 & 20.9 & 0.0 & 8.8 & 28.6 & 0.0 & 12.3 & 44.5 & 0.0 & 30.8 \\
\bottomrule
\end{tabular}
\end{table*}

\newpage
\subsection{Full PINN Results}
\label{app:full_pinn_results}
Given the space limitations of the main text, we summarize Q2 there using
configuration-level ranges and medians. This section reports the complete Q2
results for each evaluated method--benchmark configuration, including relative
$L_2$ errors and conflict-rate transitions over the four training stages. For
the six gradient-surgery methods (excluding the Adam control), the summary
contains 36 two-loss and 24 three-loss method--benchmark configurations. The
corresponding medians reported in the main text are medians over these
configuration-level reductions.

\subsubsection{Quantitative PINN Results for Q2}
Table~\ref{tab:full_q2_results} reports the complete relative $L_2$ errors.
Across all evaluated base methods and PINN settings, adding GUA reduces the
mean error, consistent with the aggregate results reported in the main text.
Adam+GUA also improves over Adam in all ten settings despite using no
pre-optimizer gradient surgery.
\begin{table*}[h]
\centering
\caption{Complete Q2 relative $L_2$ errors (mean \(\pm\) sample standard deviation over five seeds). Here, $E$ denotes the relative $L_2$ error defined in Appendix~\ref{app:evaluation_metrics}. ``Imp.'' denotes $100(1-\bar E_{\mathrm{GUA}}/\bar E_{\mathrm{base}})$.}
\label{tab:main_results}
\label{tab:full_q2_results}
\scriptsize
\setlength{\tabcolsep}{0.58pt}
\renewcommand{\arraystretch}{1.15}
\begin{tabular}{c c c c c c c c c c c}
\toprule
\multicolumn{1}{c}{\multirow{2}{*}[-0.3ex]{\textbf{Method}}}
& \multicolumn{4}{c}{\textbf{Schr\"odinger}}
& \multicolumn{4}{c}{\textbf{Burgers}}
& \multicolumn{2}{c}{\textbf{Kovasznay}} \\
\cmidrule(lr){2-5} \cmidrule(lr){6-9} \cmidrule(lr){10-11}
& 2-loss & \imphead & 3-loss & \imphead & 2-loss & \imphead & 3-loss & \imphead & 2-loss & \imphead \\
\midrule
\rowcolor{gray!14} Adam & \qtwoentry{3.15e-2}{1.10e-2} & \imprsingle{--} & \qtwoentry{3.15e-2}{1.10e-2} & \imprsingle{--} & \qtwoentry{4.55e-3}{1.03e-3} & \imprsingle{--} & \qtwoentry{3.07e-3}{1.37e-3} & \imprsingle{--} & \qtwoentry{4.69e-4}{7.44e-5} & \imprsingle{--} \\
\rowcolor{gray!14} \textbf{+GUA} & \textbf{\qtwoentry{1.74e-2}{4.19e-3}} & \imprsingle{44.8} & \textbf{\qtwoentry{2.42e-2}{6.59e-3}} & \imprsingle{23.2} & \textbf{\qtwoentry{1.18e-3}{5.86e-4}} & \imprsingle{74.1} & \textbf{\qtwoentry{7.77e-4}{3.35e-4}} & \imprsingle{74.7} & \textbf{\qtwoentry{3.46e-4}{1.53e-5}} & \imprsingle{26.2}\\
\addlinespace[1.5pt]
PCGrad & \qtwoentry{2.51e-2}{7.78e-3} & \imprsingle{--} & \qtwoentry{2.12e-2}{5.21e-3} & \imprsingle{--} & \qtwoentry{3.05e-3}{7.31e-4} & \imprsingle{--} & \qtwoentry{2.02e-3}{8.11e-4} & \imprsingle{--} & \qtwoentry{4.82e-4}{1.02e-4} & \imprsingle{--} \\
\textbf{+GUA} & \textbf{\qtwoentry{1.46e-2}{1.52e-3}} & \imprsingle{41.8} & \textbf{\qtwoentry{1.01e-2}{8.11e-4}} & \imprsingle{52.4} & \textbf{\qtwoentry{8.35e-4}{1.32e-4}} & \imprsingle{72.6} & \textbf{\qtwoentry{7.61e-4}{1.81e-4}} & \imprsingle{62.3} & \textbf{\qtwoentry{3.26e-4}{7.42e-5}} & \imprsingle{32.4}\\
\addlinespace[1.5pt]
\rowcolor{gray!14} CAGrad & \qtwoentry{1.40e-2}{1.53e-3} & \imprsingle{--} & \qtwoentry{2.00e-2}{5.06e-3} & \imprsingle{--} & \qtwoentry{1.70e-3}{4.36e-4} & \imprsingle{--} & \qtwoentry{2.25e-3}{4.90e-4} & \imprsingle{--} & \qtwoentry{4.78e-4}{4.76e-5} & \imprsingle{--} \\
\rowcolor{gray!14} \textbf{+GUA} & \textbf{\qtwoentry{1.18e-2}{8.77e-4}} & \imprsingle{15.7} & \textbf{\qtwoentry{1.08e-2}{1.92e-3}} & \imprsingle{46.0} & \textbf{\qtwoentry{6.97e-4}{1.26e-4}} & \imprsingle{59.0} & \textbf{\qtwoentry{7.39e-4}{5.89e-5}} & \imprsingle{67.2} & \textbf{\qtwoentry{1.11e-4}{3.12e-5}} & \imprsingle{76.8}\\
\addlinespace[1.5pt]
IMTL-G & \qtwoentry{4.12e-2}{1.09e-2} & \imprsingle{--} & \qtwoentry{5.60e-1}{3.12e-1} & \imprsingle{--} & \qtwoentry{2.51e-3}{8.24e-4} & \imprsingle{--} & \qtwoentry{5.83e-2}{3.66e-2} & \imprsingle{--} & \qtwoentry{1.79e-4}{2.83e-5} & \imprsingle{--} \\
\textbf{+GUA} & \textbf{\qtwoentry{3.21e-2}{1.86e-3}} & \imprsingle{22.1} & \textbf{\qtwoentry{1.02e-1}{5.91e-2}} & \imprsingle{81.8} & \textbf{\qtwoentry{6.82e-4}{6.76e-5}} & \imprsingle{72.8} & \textbf{\qtwoentry{5.14e-3}{1.07e-3}} & \imprsingle{91.2} & \textbf{\qtwoentry{9.25e-5}{1.03e-5}} & \imprsingle{48.3}\\
\addlinespace[1.5pt]
\rowcolor{gray!14} A-MTL & \qtwoentry{1.60e-2}{1.87e-3} & \imprsingle{--} & \qtwoentry{3.00e-2}{3.78e-3} & \imprsingle{--} & \qtwoentry{2.95e-3}{5.24e-4} & \imprsingle{--} & \qtwoentry{7.31e-3}{2.89e-3} & \imprsingle{--} & \qtwoentry{1.83e-4}{1.56e-5} & \imprsingle{--} \\
\rowcolor{gray!14} \textbf{+GUA} & \textbf{\qtwoentry{1.23e-2}{2.93e-3}} & \imprsingle{23.1} & \textbf{\qtwoentry{2.15e-2}{4.18e-3}} & \imprsingle{28.3} & \textbf{\qtwoentry{1.05e-3}{2.72e-4}} & \imprsingle{64.4} & \textbf{\qtwoentry{1.56e-3}{3.34e-4}} & \imprsingle{78.7} & \textbf{\qtwoentry{1.56e-4}{2.17e-5}} & \imprsingle{14.8}\\
\addlinespace[1.5pt]
UPGrad & \qtwoentry{2.56e-2}{7.20e-3} & \imprsingle{--} & \qtwoentry{1.98e-2}{4.10e-3} & \imprsingle{--} & \qtwoentry{3.03e-3}{4.96e-4} & \imprsingle{--} & \qtwoentry{3.08e-3}{1.46e-3} & \imprsingle{--} & \qtwoentry{4.49e-4}{6.56e-5} & \imprsingle{--} \\
\textbf{+GUA} & \textbf{\qtwoentry{1.52e-2}{2.12e-3}} & \imprsingle{40.6} & \textbf{\qtwoentry{1.10e-2}{1.10e-3}} & \imprsingle{44.4} & \textbf{\qtwoentry{7.68e-4}{1.12e-4}} & \imprsingle{74.7} & \textbf{\qtwoentry{9.69e-4}{1.80e-4}} & \imprsingle{68.5} & \textbf{\qtwoentry{3.23e-4}{7.83e-5}} & \imprsingle{28.1}\\
\addlinespace[1.5pt]
\rowcolor{gray!14} ConFIG & \qtwoentry{1.89e-2}{5.46e-3} & \imprsingle{--} & \qtwoentry{2.33e-2}{7.89e-3} & \imprsingle{--} & \qtwoentry{1.74e-3}{3.38e-4} & \imprsingle{--} & \qtwoentry{3.48e-3}{8.30e-4} & \imprsingle{--} & \qtwoentry{1.86e-4}{4.02e-5} & \imprsingle{--} \\
\rowcolor{gray!14} \textbf{+GUA} & \textbf{\qtwoentry{1.38e-2}{2.35e-3}} & \imprsingle{27.0} & \textbf{\qtwoentry{1.33e-2}{1.69e-3}} & \imprsingle{42.9} & \textbf{\qtwoentry{6.50e-4}{1.27e-4}} & \imprsingle{62.6} & \textbf{\qtwoentry{1.34e-3}{1.31e-4}} & \imprsingle{61.5} & \textbf{\qtwoentry{9.84e-5}{1.90e-5}} & \imprsingle{47.1}\\
\bottomrule
\end{tabular}

\vspace{2pt}

\begin{tabular}{c c c c c c c c c c c}
\toprule
\multicolumn{1}{c}{\multirow{2}{*}[-0.3ex]{\textbf{Method}}}
& \multicolumn{4}{c}{\textbf{Heat-MS}}
& \multicolumn{4}{c}{\textbf{Beltrami}}
& \multicolumn{2}{c}{\textbf{Poisson-5D}} \\
\cmidrule(lr){2-5} \cmidrule(lr){6-9} \cmidrule(lr){10-11}
& 2-loss & \imphead & 3-loss & \imphead & 2-loss & \imphead & 3-loss & \imphead & 2-loss & \imphead \\
\midrule
\rowcolor{gray!14} Adam & \qtwoentry{1.73e-2}{4.58e-3} & \imprsingle{--} & \qtwoentry{1.57e-1}{6.51e-2} & \imprsingle{--} & \qtwoentry{6.50e-3}{3.14e-4} & \imprsingle{--} & \qtwoentry{5.26e-3}{2.84e-4} & \imprsingle{--} & \qtwoentry{5.95e-4}{3.22e-5} & \imprsingle{--} \\
\rowcolor{gray!14} \textbf{+GUA} & \textbf{\qtwoentry{1.46e-2}{1.35e-3}} & \imprsingle{15.6} & \textbf{\qtwoentry{3.47e-2}{7.15e-3}} & \imprsingle{77.9} & \textbf{\qtwoentry{5.06e-3}{1.76e-4}} & \imprsingle{22.2} & \textbf{\qtwoentry{3.90e-3}{2.23e-4}} & \imprsingle{25.9} & \textbf{\qtwoentry{3.94e-4}{3.25e-5}} & \imprsingle{33.8}\\
\addlinespace[1.5pt]
PCGrad & \qtwoentry{2.05e-2}{4.36e-3} & \imprsingle{--} & \qtwoentry{3.40e-1}{1.34e-1} & \imprsingle{--} & \qtwoentry{6.70e-3}{4.75e-4} & \imprsingle{--} & \qtwoentry{5.26e-3}{2.96e-4} & \imprsingle{--} & \qtwoentry{5.88e-4}{4.29e-5} & \imprsingle{--} \\
\textbf{+GUA} & \textbf{\qtwoentry{1.35e-2}{4.40e-3}} & \imprsingle{34.1} & \textbf{\qtwoentry{4.94e-2}{1.43e-2}} & \imprsingle{85.5} & \textbf{\qtwoentry{5.22e-3}{3.12e-4}} & \imprsingle{22.1} & \textbf{\qtwoentry{4.18e-3}{3.07e-4}} & \imprsingle{20.5} & \textbf{\qtwoentry{3.38e-4}{2.34e-5}} & \imprsingle{42.5}\\
\addlinespace[1.5pt]
\rowcolor{gray!14} CAGrad & \qtwoentry{3.19e-2}{8.27e-3} & \imprsingle{--} & \qtwoentry{6.26e-1}{4.33e-2} & \imprsingle{--} & \qtwoentry{4.13e-3}{2.69e-4} & \imprsingle{--} & \qtwoentry{3.91e-3}{1.78e-4} & \imprsingle{--} & \qtwoentry{4.28e-4}{3.02e-4} & \imprsingle{--} \\
\rowcolor{gray!14} \textbf{+GUA} & \textbf{\qtwoentry{1.13e-2}{3.08e-3}} & \imprsingle{64.6} & \textbf{\qtwoentry{1.46e-1}{4.81e-2}} & \imprsingle{76.7} & \textbf{\qtwoentry{3.10e-3}{1.68e-4}} & \imprsingle{24.9} & \textbf{\qtwoentry{2.96e-3}{9.99e-5}} & \imprsingle{24.3} & \textbf{\qtwoentry{1.37e-4}{2.07e-5}} & \imprsingle{68.0}\\
\addlinespace[1.5pt]
IMTL-G & \qtwoentry{8.82e-1}{2.26e-1} & \imprsingle{--} & \qtwoentry{1.00e+0}{5.56e-4} & \imprsingle{--} & \qtwoentry{3.83e-3}{2.80e-4} & \imprsingle{--} & \qtwoentry{6.58e-3}{1.73e-3} & \imprsingle{--} & \qtwoentry{4.94e-4}{1.94e-4} & \imprsingle{--} \\
\textbf{+GUA} & \textbf{\qtwoentry{1.63e-2}{3.28e-3}} & \imprsingle{98.2} & \textbf{\qtwoentry{3.26e-1}{1.53e-1}} & \imprsingle{67.4} & \textbf{\qtwoentry{2.92e-3}{1.76e-4}} & \imprsingle{23.8} & \textbf{\qtwoentry{3.19e-3}{2.29e-4}} & \imprsingle{51.5} & \textbf{\qtwoentry{1.44e-4}{9.57e-6}} & \imprsingle{70.9}\\
\addlinespace[1.5pt]
\rowcolor{gray!14} A-MTL & \qtwoentry{7.61e-2}{3.73e-2} & \imprsingle{--} & \qtwoentry{7.71e-1}{3.92e-2} & \imprsingle{--} & \qtwoentry{4.06e-3}{2.26e-4} & \imprsingle{--} & \qtwoentry{3.19e-3}{2.05e-4} & \imprsingle{--} & \qtwoentry{5.16e-4}{2.59e-4} & \imprsingle{--} \\
\rowcolor{gray!14} \textbf{+GUA} & \textbf{\qtwoentry{2.31e-2}{7.88e-3}} & \imprsingle{69.6} & \textbf{\qtwoentry{1.47e-1}{5.10e-2}} & \imprsingle{80.9} & \textbf{\qtwoentry{2.99e-3}{2.01e-4}} & \imprsingle{26.4} & \textbf{\qtwoentry{2.82e-3}{1.18e-4}} & \imprsingle{11.6} & \textbf{\qtwoentry{2.93e-4}{1.93e-4}} & \imprsingle{43.2}\\
\addlinespace[1.5pt]
UPGrad & \qtwoentry{2.05e-2}{4.71e-3} & \imprsingle{--} & \qtwoentry{3.20e-1}{1.93e-1} & \imprsingle{--} & \qtwoentry{6.73e-3}{3.70e-4} & \imprsingle{--} & \qtwoentry{5.15e-3}{2.98e-4} & \imprsingle{--} & \qtwoentry{5.93e-4}{3.44e-5} & \imprsingle{--} \\
\textbf{+GUA} & \textbf{\qtwoentry{1.59e-2}{1.36e-3}} & \imprsingle{22.4} & \textbf{\qtwoentry{4.62e-2}{8.30e-3}} & \imprsingle{85.6} & \textbf{\qtwoentry{5.14e-3}{3.32e-4}} & \imprsingle{23.6} & \textbf{\qtwoentry{4.26e-3}{1.79e-4}} & \imprsingle{17.3} & \textbf{\qtwoentry{3.52e-4}{2.73e-5}} & \imprsingle{40.6}\\
\addlinespace[1.5pt]
\rowcolor{gray!14} ConFIG & \qtwoentry{5.77e-2}{1.31e-2} & \imprsingle{--} & \qtwoentry{8.16e-1}{2.98e-2} & \imprsingle{--} & \qtwoentry{4.01e-3}{2.75e-4} & \imprsingle{--} & \qtwoentry{4.11e-3}{1.45e-4} & \imprsingle{--} & \qtwoentry{3.16e-4}{2.94e-4} & \imprsingle{--} \\
\rowcolor{gray!14} \textbf{+GUA} & \textbf{\qtwoentry{9.85e-3}{1.56e-3}} & \imprsingle{82.9} & \textbf{\qtwoentry{1.45e-1}{9.18e-2}} & \imprsingle{82.2} & \textbf{\qtwoentry{2.92e-3}{1.80e-4}} & \imprsingle{27.2} & \textbf{\qtwoentry{2.82e-3}{1.82e-4}} & \imprsingle{31.4} & \textbf{\qtwoentry{1.40e-4}{2.76e-5}} & \imprsingle{55.7}\\
\bottomrule
\end{tabular}
\end{table*}

\newpage
\subsubsection{Conflict-Rate PINN Results for Q2}

Table~\ref{tab:pinn_conflict_rates} reports the corresponding
$R_g\!\rightarrow R_a\!\rightarrow R_u\!\rightarrow R_p$ transitions.
Unlike Q1, which fixes ConFIG and varies the optimizer, these experiments fix
Adam and vary the gradient surgery method. The resulting transitions show that
GUM is not specific to a particular pre-optimizer construction. Different
methods produce different $R_a$ and exhibit different responses to the same
optimizer transformation.
When $R_a=0$, any nonzero $R_u$ directly indicates strict GUM. When residual
pre-optimizer conflicts remain, an increase from $R_a$ to $R_u$ indicates
generalized GUM, showing that Adam further increases the frequency of
conflicting proposals. Despite these differences across gradient surgery
methods, GUA yields $R_p=0$ in every evaluated setting, so the applied update is
conflict-free after alignment. These transitions show that update-level
alignment remains effective across different pre-optimizer constructions and
connect the observed accuracy gains to the intended mechanism.
\begin{table*}[h]
\centering
\caption{Four-stage conflict rates (\%) for the PINN Q2 experiments, averaged over five seeds. }
\label{tab:pinn_conflict_rates}
\scriptsize
\setlength{\tabcolsep}{1.0pt}
\renewcommand{\arraystretch}{1.15}
\begin{tabular}{c *{4}{w{c}{15pt}} @{} w{c}{7pt} @{} *{4}{w{c}{15pt}} @{} w{c}{7pt} @{} *{4}{w{c}{15pt}} @{} w{c}{7pt} @{} *{4}{w{c}{15pt}} @{} w{c}{7pt} @{} *{4}{w{c}{15pt}}}
\toprule
\multicolumn{1}{c}{\multirow{3}{*}[-0.5ex]{\textbf{Method}}}
& \multicolumn{9}{c}{\textbf{Schr\"odinger}} & \multicolumn{1}{c}{}
& \multicolumn{9}{c}{\textbf{Burgers}} & \multicolumn{1}{c}{}
& \multicolumn{4}{c}{\textbf{Kovasznay}} \\
\cmidrule(lr){2-10} \cmidrule(lr){12-20} \cmidrule(lr){22-25}
& \multicolumn{4}{c}{\textbf{2-loss}} & & \multicolumn{4}{c}{\textbf{3-loss}} & & \multicolumn{4}{c}{\textbf{2-loss}} & & \multicolumn{4}{c}{\textbf{3-loss}} & & \multicolumn{4}{c}{\textbf{2-loss}} \\
\cmidrule(lr){2-5} \cmidrule(lr){7-10} \cmidrule(lr){12-15} \cmidrule(lr){17-20} \cmidrule(lr){22-25}
& $R_g$ & $R_a$ & $R_u$ & $R_p$ & & $R_g$ & $R_a$ & $R_u$ & $R_p$ & & $R_g$ & $R_a$ & $R_u$ & $R_p$ & & $R_g$ & $R_a$ & $R_u$ & $R_p$ & & $R_g$ & $R_a$ & $R_u$ & $R_p$ \\
\midrule
\rowcolor{gray!14} Adam & 44.3 & 15.9 & 13.2 & -- & & 65.2 & 38.2 & 35.2 & -- & & 28.5 & 20.0 & 22.6 & -- & & 55.4 & 36.1 & 34.0 & -- & & 3.0 & 1.3 & 1.2 & -- \\
\rowcolor{gray!14} \textbf{+GUA} & 36.4 & 13.1 & 13.6 & 0.0 & & 61.5 & 35.6 & 42.1 & 0.0 & & 20.0 & 14.2 & 11.6 & 0.0 & & 35.9 & 23.2 & 15.9 & 0.0 & & 0.9 & 0.3 & 0.3 & 0.0 \\
\addlinespace[1.5pt]
PCGrad & 36.3 & 0.0 & 6.3 & -- & & 51.9 & 0.3 & 16.0 & -- & & 35.1 & 0.0 & 31.6 & -- & & 56.6 & 4.7 & 45.0 & -- & & 2.2 & 0.0 & 0.6 & -- \\
\textbf{+GUA} & 32.3 & 0.0 & 5.1 & 0.0 & & 49.6 & 0.4 & 16.2 & 0.0 & & 33.8 & 0.0 & 29.9 & 0.0 & & 63.8 & 4.8 & 51.7 & 0.0 & & 1.1 & 0.0 & 0.2 & 0.0 \\
\addlinespace[1.5pt]
\rowcolor{gray!14} CAGrad & 35.0 & 5.1 & 15.4 & -- & & 43.7 & 7.4 & 23.0 & -- & & 31.1 & 3.8 & 27.1 & -- & & 56.1 & 12.4 & 42.5 & -- & & 6.1 & 1.4 & 7.9 & -- \\
\rowcolor{gray!14} \textbf{+GUA} & 29.5 & 1.2 & 7.1 & 0.0 & & 41.9 & 3.5 & 15.8 & 0.0 & & 29.0 & 3.8 & 21.6 & 0.0 & & 63.2 & 16.3 & 43.0 & 0.0 & & 1.9 & 0.1 & 6.6 & 0.0 \\
\addlinespace[1.5pt]
IMTL-G & 67.8 & 0.0 & 33.2 & -- & & 94.4 & 16.7 & 53.4 & -- & & 41.9 & 0.0 & 23.7 & -- & & 94.5 & 1.7 & 55.6 & -- & & 3.6 & 0.0 & 4.8 & -- \\
\textbf{+GUA} & 52.8 & 0.0 & 22.5 & 0.0 & & 86.8 & 2.3 & 44.2 & 0.0 & & 33.5 & 0.0 & 15.7 & 0.0 & & 84.3 & 3.1 & 49.0 & 0.0 & & 0.9 & 0.0 & 2.9 & 0.0 \\
\addlinespace[1.5pt]
\rowcolor{gray!14} A-MTL & 33.6 & 0.2 & 18.3 & -- & & 57.7 & 4.5 & 37.5 & -- & & 27.5 & 0.8 & 34.2 & -- & & 66.8 & 9.4 & 57.2 & -- & & 5.0 & 0.1 & 11.5 & -- \\
\rowcolor{gray!14} \textbf{+GUA} & 27.2 & 0.1 & 15.9 & 0.0 & & 44.5 & 0.6 & 29.4 & 0.0 & & 30.2 & 1.1 & 35.7 & 0.0 & & 63.7 & 4.9 & 57.4 & 0.0 & & 2.2 & 0.0 & 9.5 & 0.0 \\
\addlinespace[1.5pt]
UPGrad & 35.7 & 0.4 & 6.2 & -- & & 50.5 & 3.8 & 13.2 & -- & & 33.5 & 4.7 & 29.0 & -- & & 54.0 & 8.8 & 41.1 & -- & & 2.3 & 0.0 & 0.7 & -- \\
\textbf{+GUA} & 31.6 & 0.4 & 4.8 & 0.0 & & 50.4 & 5.2 & 14.8 & 0.0 & & 32.6 & 7.4 & 28.2 & 0.0 & & 60.3 & 18.5 & 47.4 & 0.0 & & 1.1 & 0.0 & 0.2 & 0.0 \\
\addlinespace[1.5pt]
\rowcolor{gray!14} ConFIG & 40.8 & 0.0 & 23.7 & -- & & 56.3 & 0.0 & 35.2 & -- & & 30.5 & 0.0 & 29.4 & -- & & 73.7 & 0.0 & 57.4 & -- & & 5.5 & 0.0 & 6.4 & -- \\
\rowcolor{gray!14} \textbf{+GUA} & 33.8 & 0.0 & 13.9 & 0.0 & & 50.7 & 0.0 & 27.6 & 0.0 & & 30.0 & 0.0 & 24.2 & 0.0 & & 73.0 & 0.0 & 55.3 & 0.0 & & 1.5 & 0.0 & 6.1 & 0.0 \\
\bottomrule
\end{tabular}

\vspace{2pt}

\begin{tabular}{c *{4}{w{c}{15pt}} @{} w{c}{7pt} @{} *{4}{w{c}{15pt}} @{} w{c}{7pt} @{} *{4}{w{c}{15pt}} @{} w{c}{7pt} @{} *{4}{w{c}{15pt}} @{} w{c}{7pt} @{} *{4}{w{c}{15pt}}}
\toprule
\multicolumn{1}{c}{\multirow{3}{*}[-0.5ex]{\textbf{Method}}}
& \multicolumn{9}{c}{\textbf{Heat-MS}} & \multicolumn{1}{c}{}
& \multicolumn{9}{c}{\textbf{Beltrami}} & \multicolumn{1}{c}{}
& \multicolumn{4}{c}{\textbf{Poisson-5D}} \\
\cmidrule(lr){2-10} \cmidrule(lr){12-20} \cmidrule(lr){22-25}
& \multicolumn{4}{c}{\textbf{2-loss}} & & \multicolumn{4}{c}{\textbf{3-loss}} & & \multicolumn{4}{c}{\textbf{2-loss}} & & \multicolumn{4}{c}{\textbf{3-loss}} & & \multicolumn{4}{c}{\textbf{2-loss}} \\
\cmidrule(lr){2-5} \cmidrule(lr){7-10} \cmidrule(lr){12-15} \cmidrule(lr){17-20} \cmidrule(lr){22-25}
& $R_g$ & $R_a$ & $R_u$ & $R_p$ & & $R_g$ & $R_a$ & $R_u$ & $R_p$ & & $R_g$ & $R_a$ & $R_u$ & $R_p$ & & $R_g$ & $R_a$ & $R_u$ & $R_p$ & & $R_g$ & $R_a$ & $R_u$ & $R_p$ \\
\midrule
\rowcolor{gray!14} Adam & 14.0 & 7.5 & 4.8 & -- & & 72.8 & 40.1 & 54.1 & -- & & 1.6 & 0.2 & 0.2 & -- & & 4.2 & 0.6 & 0.5 & -- & & 41.3 & 7.2 & 3.8 & -- \\
\rowcolor{gray!14} \textbf{+GUA} & 17.6 & 6.8 & 2.4 & 0.0 & & 32.9 & 17.7 & 19.9 & 0.0 & & 2.0 & 0.2 & 0.8 & 0.0 & & 3.6 & 0.4 & 1.6 & 0.0 & & 41.3 & 4.5 & 3.7 & 0.0 \\
\addlinespace[1.5pt]
PCGrad & 14.9 & 0.0 & 5.1 & -- & & 77.6 & 6.2 & 57.7 & -- & & 1.2 & 0.0 & 0.2 & -- & & 3.6 & 0.1 & 0.3 & -- & & 24.1 & 0.0 & 3.8 & -- \\
\textbf{+GUA} & 14.6 & 0.0 & 2.2 & 0.0 & & 34.6 & 0.7 & 17.9 & 0.0 & & 1.7 & 0.0 & 0.6 & 0.0 & & 3.6 & 0.0 & 1.1 & 0.0 & & 42.4 & 0.0 & 4.3 & 0.0 \\
\addlinespace[1.5pt]
\rowcolor{gray!14} CAGrad & 20.4 & 1.7 & 14.3 & -- & & 78.9 & 21.8 & 60.2 & -- & & 5.8 & 0.3 & 4.9 & -- & & 20.2 & 0.9 & 10.8 & -- & & 48.5 & 5.1 & 22.9 & -- \\
\rowcolor{gray!14} \textbf{+GUA} & 20.1 & 1.6 & 14.2 & 0.0 & & 37.8 & 2.8 & 21.0 & 0.0 & & 7.9 & 0.1 & 5.6 & 0.0 & & 13.6 & 0.1 & 6.7 & 0.0 & & 35.9 & 0.2 & 14.0 & 0.0 \\
\addlinespace[1.5pt]
IMTL-G & 53.7 & 0.0 & 30.6 & -- & & 88.5 & 5.4 & 18.7 & -- & & 5.1 & 0.0 & 5.9 & -- & & 89.5 & 0.0 & 28.1 & -- & & 51.0 & 0.0 & 33.8 & -- \\
\textbf{+GUA} & 28.9 & 0.0 & 13.4 & 0.0 & & 60.9 & 4.1 & 36.3 & 0.0 & & 4.1 & 0.0 & 4.4 & 0.0 & & 24.5 & 0.1 & 10.7 & 0.0 & & 44.1 & 0.0 & 26.2 & 0.0 \\
\addlinespace[1.5pt]
\rowcolor{gray!14} A-MTL & 23.7 & 0.1 & 17.7 & -- & & 75.4 & 15.0 & 64.1 & -- & & 14.5 & 0.0 & 14.8 & -- & & 28.2 & 0.0 & 19.9 & -- & & 48.4 & 0.4 & 30.8 & -- \\
\rowcolor{gray!14} \textbf{+GUA} & 26.2 & 0.3 & 11.8 & 0.0 & & 55.8 & 3.5 & 27.9 & 0.0 & & 10.2 & 0.0 & 9.7 & 0.0 & & 14.3 & 0.0 & 13.0 & 0.0 & & 42.3 & 0.0 & 25.0 & 0.0 \\
\addlinespace[1.5pt]
UPGrad & 13.1 & 0.4 & 4.7 & -- & & 78.3 & 12.3 & 55.5 & -- & & 1.2 & 0.0 & 0.2 & -- & & 3.3 & 0.0 & 0.3 & -- & & 23.7 & 0.0 & 3.9 & -- \\
\textbf{+GUA} & 15.4 & 0.4 & 2.4 & 0.0 & & 33.8 & 4.8 & 16.1 & 0.0 & & 1.8 & 0.0 & 0.7 & 0.0 & & 4.5 & 0.0 & 1.2 & 0.0 & & 42.2 & 0.0 & 4.2 & 0.0 \\
\addlinespace[1.5pt]
\rowcolor{gray!14} ConFIG & 22.8 & 0.0 & 12.8 & -- & & 81.3 & 0.0 & 64.6 & -- & & 13.1 & 0.0 & 11.6 & -- & & 27.5 & 0.0 & 16.8 & -- & & 52.6 & 0.0 & 30.5 & -- \\
\rowcolor{gray!14} \textbf{+GUA} & 27.7 & 0.0 & 9.7 & 0.0 & & 48.6 & 0.0 & 26.6 & 0.0 & & 7.4 & 0.0 & 7.6 & 0.0 & & 12.8 & 0.0 & 8.8 & 0.0 & & 44.1 & 0.0 & 22.3 & 0.0 \\
\bottomrule
\end{tabular}
\end{table*}

\newpage
\subsection{Performance across Optimizers}
\label{app:gua_cross_optimizer}

The main experiments use Adam as the default optimizer. We further evaluate GUA with three representative optimizers from Q1: RMSProp for coordinate-wise adaptive scaling, AdamW for historical state, adaptive scaling, and decoupled weight decay, and SOAP for second-order preconditioning. SOAP is particularly relevant because its preconditioning has been shown to promote gradient alignment and mitigate conflicts in PINNs~\cite{wang2025gradient}. For each optimizer $\mathcal O$, we compare $\text{ConFIG}+\mathcal O$ with $\text{ConFIG}+\mathcal O+\text{GUA}$. Since these optimizers differ in transformation geometry and internal state, we adapt the projection metric and State Alignment strategy accordingly.

At each step, the optimizer produces a proposal $u_t$, which GUA projects onto
the conflict-free cone under the time-dependent metric $M_t$:
\[
\min_{d_t}\;\frac{1}{2}(d_t-u_t)^\top M_t(d_t-u_t)
\qquad
\text{s.t.}\quad G_t^\top d_t\ge 0.
\]
We use the second-moment geometry induced by each optimizer rather than a common Euclidean metric. Specifically,
$
M_{\mathrm{RMSProp},t}\approx\operatorname{diag}(\sqrt{v_t}+\epsilon),
M_{\mathrm{AdamW},t}\approx\operatorname{diag}(\sqrt{\hat v_t}+\epsilon).
$
For SOAP, the proposal and task gradients are represented in the current Shampoo coordinate system, where
$M_{\mathrm{SOAP},t}\approx\operatorname{diag}(\sqrt{\tilde v_t}+\epsilon)$,
and the projected update is then mapped back to parameter space. The dual solver uses
$K_t=G_t^\top M_t^{-1}G_t$ and reconstructs
$p_t=u_t+M_t^{-1}G_t\lambda_t$.

After applying the projected update, GUA softly aligns the optimizer state toward targets reconstructed from the applied update. RMSProp aligns its second-moment state with $(\rho_m,\rho_v)=(0,0.5)$. AdamW aligns its first- and second-moment states with $(0.1,0.03)$ while handling decoupled weight decay separately. SOAP aligns the exponential moving averages that determine its update with $(0.03,0.003)$, while keeping the Shampoo preconditioner and coordinate basis unchanged. State Alignment is applied only when GUA modifies the optimizer proposal.
RMSProp uses a learning rate of $10^{-4}/10^{-4}$, whereas AdamW and SOAP use $10^{-3}/10^{-4}$.

As shown in Figure~\ref{fig:gua_cross_optimizer}, GUA improves the displayed
settings across optimizer families. For Burgers, it reduces the relative $L_2$
error of RMSProp, AdamW, and SOAP by $52.35\%$, $65.19\%$, and $13.78\%$ in
the two-loss setting, and by $58.39\%$, $62.92\%$, and $42.73\%$ in the
three-loss setting.
For Schr\"odinger, the corresponding reductions are $22.98\%$, $40.00\%$, and
$1.70\%$ in the two-loss setting, and $47.50\%$, $43.66\%$, and $4.77\%$ in
the three-loss setting, respectively.
Notably, GUA further improves SOAP even though its second-order preconditioning
already promotes gradient alignment. This suggests that explicit update-level
alignment can provide complementary benefits beyond the alignment induced by
the optimizer's preconditioning geometry. Together with the Adam results in Q2, these results suggest that
GUA extends beyond Adam to optimizers involving adaptive scaling, decoupled
weight decay, and preconditioning.

\begin{figure*}[h]
\centering
\includegraphics[width=\textwidth]{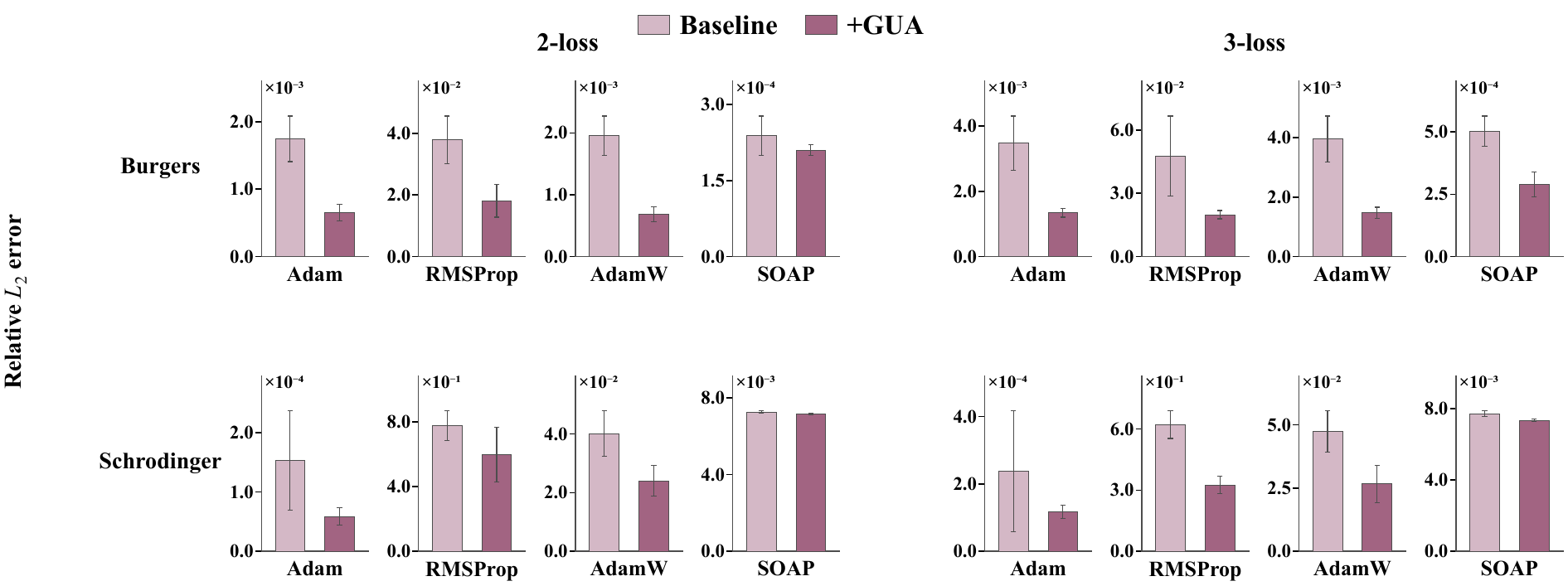}
\caption{Relative $L_2$ error across optimizer families for the Burgers and Schrödinger equations.}
\label{fig:gua_cross_optimizer}
\end{figure*}

\newpage
\subsection{Mechanism Diagnostics}
\label{app:mechanism_diagnostics}
\subsubsection{Same-State Norm-Matched Comparison}
\label{app:same_state_counterfactual_details}

Table~\ref{tab:q3_equation_win_rates} reports the per-benchmark results for the
same-state, norm-matched comparison introduced in Q3. At each evaluated step,
the parameters, optimizer state, batch, and learning rate are held fixed while
comparing the optimizer proposal $u_t$, the aligned update $p_t$, and the
norm-matched control
\[
q_t=\frac{\|p_t\|_2}{\|u_t\|_2}u_t.
\]
All reported samples satisfy $p_t\neq0$. We report the pairwise win rates
$\Pr[W_t^{\mathrm{rel}}(p_t)<W_t^{\mathrm{rel}}(u_t)]$ and
$\Pr[W_t^{\mathrm{rel}}(p_t)<W_t^{\mathrm{rel}}(q_t)]$.
Training states are first aggregated within each run before statistics are
computed across seeds.

\begin{table}[H]
\centering
\caption{Per-benchmark Q3 win rates under two- and three-loss decompositions.}
\label{tab:q3_equation_win_rates}
\normalsize
\setlength{\tabcolsep}{5.5pt}
\begin{tabular}{ccccc}
\toprule
\multicolumn{1}{c}{\multirow{2}{*}[-0.3ex]{\textbf{Equation}}}
& \multicolumn{2}{c}{\textbf{2-loss}} & \multicolumn{2}{c}{\textbf{3-loss}} \\
\cmidrule(lr){2-3} \cmidrule(lr){4-5}
& $p_t<u_t$ (\%) & $p_t<q_t$ (\%) & $p_t<u_t$ (\%) & $p_t<q_t$ (\%) \\
\midrule
Schr\"odinger & 95.81 & 92.68 & 96.26 & 93.53 \\
Burgers       & 96.55 & 95.17 & 94.57 & 93.55 \\
Heat-MS       & 97.92 & 96.54 & 93.16 & 92.47 \\
Beltrami      & 94.77 & 93.38 & 94.05 & 93.00 \\
Kovasznay     & 98.25 & 97.54 & -- & -- \\
Poisson-5D    & 89.94 & 85.77 & -- & -- \\
\midrule
Overall       & 95.54 & 93.51 & 94.51 & 93.14 \\
\bottomrule
\end{tabular}
\end{table}

\subsubsection{Representative Gradient-Direction Slices}
\label{app:gradient_direction_slices}

To make the update-level geometry concrete, Figure~\ref{fig:gradient_direction_slices}
visualizes representative two-loss training steps across the six PDE benchmarks.
For ease of visualization, we rotate each directional slice so that one
loss-specific gradient, $g_{1,t}$, is fixed along the positive horizontal axis.
Specifically, for each full-space vector $x$, we orthogonally project it onto
\(\operatorname{span}(g_{1,t},g_{2,t})\) and express the projection in the
orthonormal basis \(e_1=g_{1,t}/\|g_{1,t}\|_2\) and
\(e_2=(g_{2,t}-\langle g_{2,t},e_1\rangle e_1)/
\|g_{2,t}-\langle g_{2,t},e_1\rangle e_1\|_2\), choosing the sign of $e_2$
so that $g_{2,t}$ has a nonnegative second coordinate. We then normalize the
displayed vectors and therefore ignore their original
magnitudes. The figure is intended to show only their relative directions and
the resulting conflict-free geometry.
Each slice shows the two loss-specific gradient directions, the constructed
direction $a_t$, the optimizer proposal $u_t$, and the aligned update $p_t$.
The shaded region denotes the corresponding conflict-free cone
$\mathcal C_t$. These examples illustrate how optimizer transformation can move
a conflict-free constructed direction outside $\mathcal C_t$, while GUA
realigns the resulting proposal with the current conflict-free cone.

\begin{figure*}[t]
\centering
\includegraphics[width=0.88\textwidth,height=0.024\textheight,keepaspectratio]{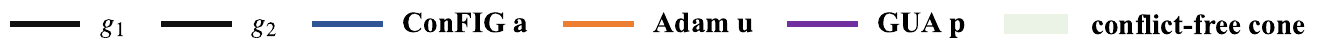}
\begin{minipage}[t]{0.98\textwidth}
\centering
\includegraphics[width=\linewidth,height=0.135\textheight,keepaspectratio]{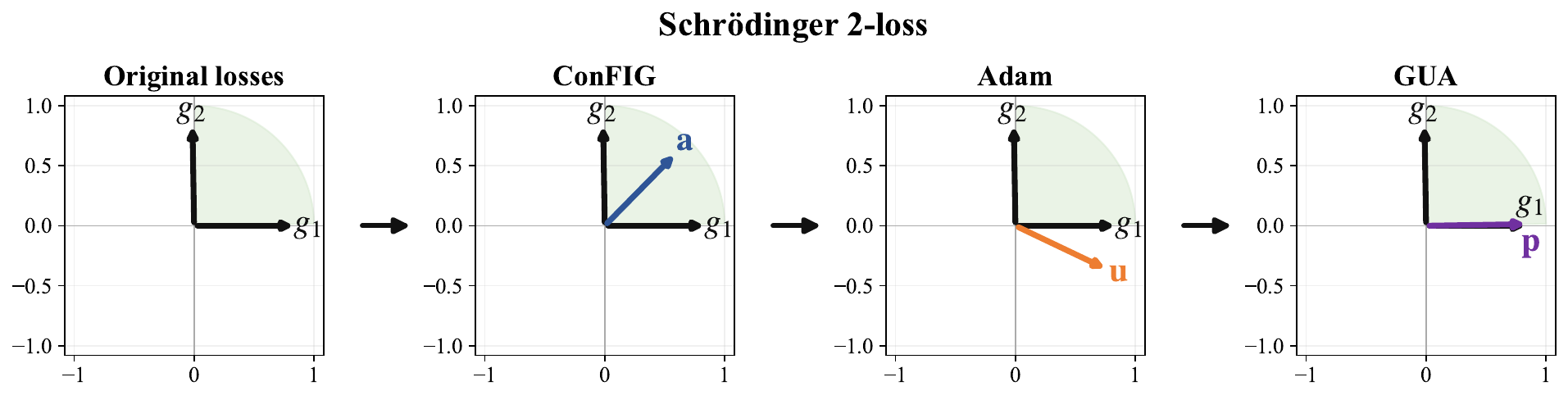}
\end{minipage}
\par\vspace{0.25em}
\begin{minipage}[t]{0.98\textwidth}
\centering
\includegraphics[width=\linewidth,height=0.135\textheight,keepaspectratio]{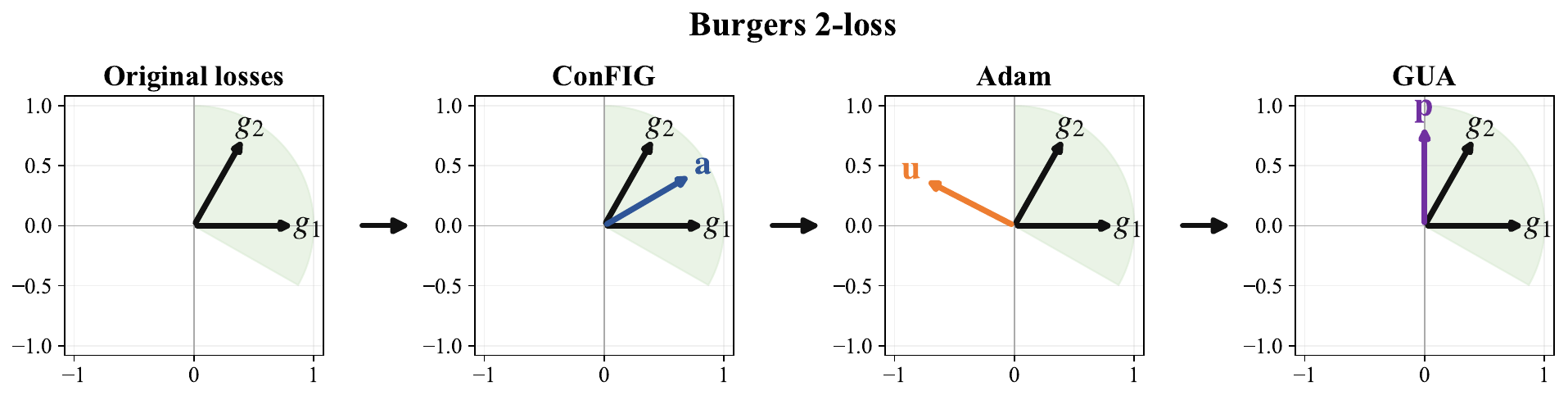}
\end{minipage}
\par\vspace{0.25em}
\begin{minipage}[t]{0.98\textwidth}
\centering
\includegraphics[width=\linewidth,height=0.135\textheight,keepaspectratio]{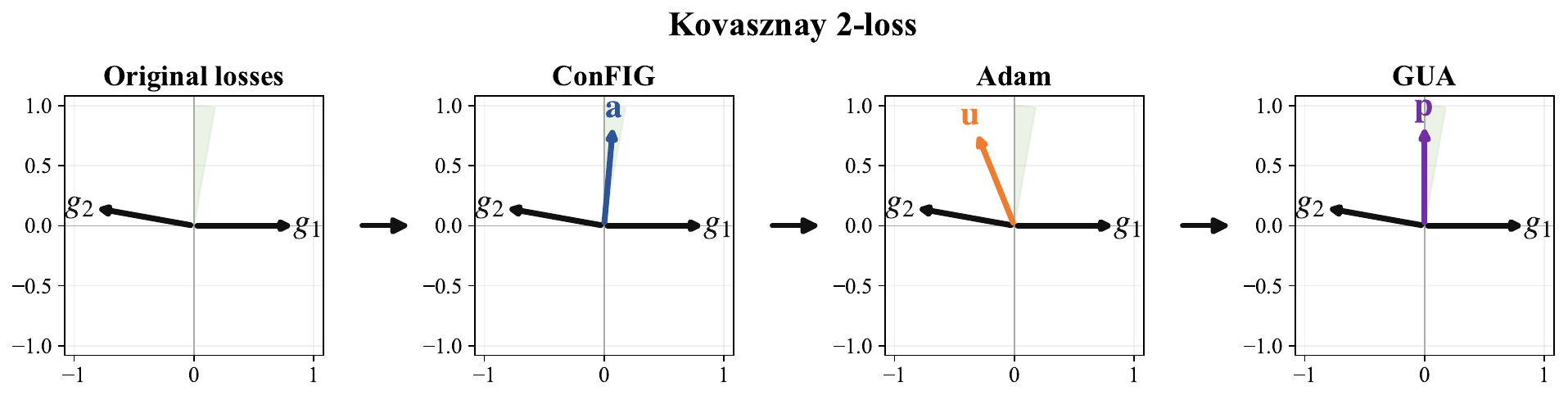}
\end{minipage}
\par\vspace{0.25em}
\begin{minipage}[t]{0.98\textwidth}
\centering
\includegraphics[width=\linewidth,height=0.135\textheight,keepaspectratio]{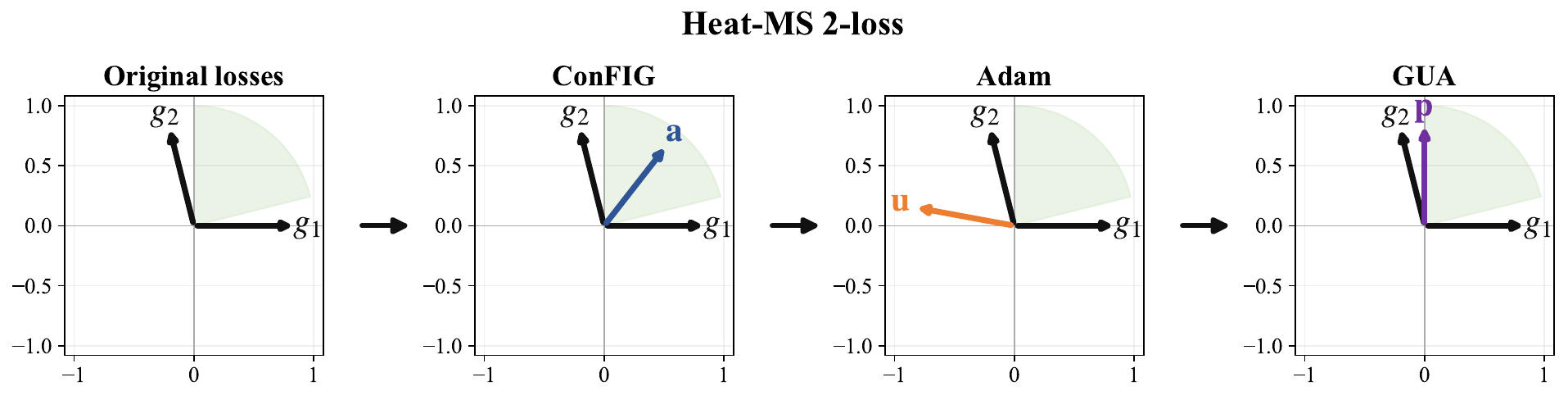}
\end{minipage}
\par\vspace{0.25em}
\begin{minipage}[t]{0.98\textwidth}
\centering
\includegraphics[width=\linewidth,height=0.135\textheight,keepaspectratio]{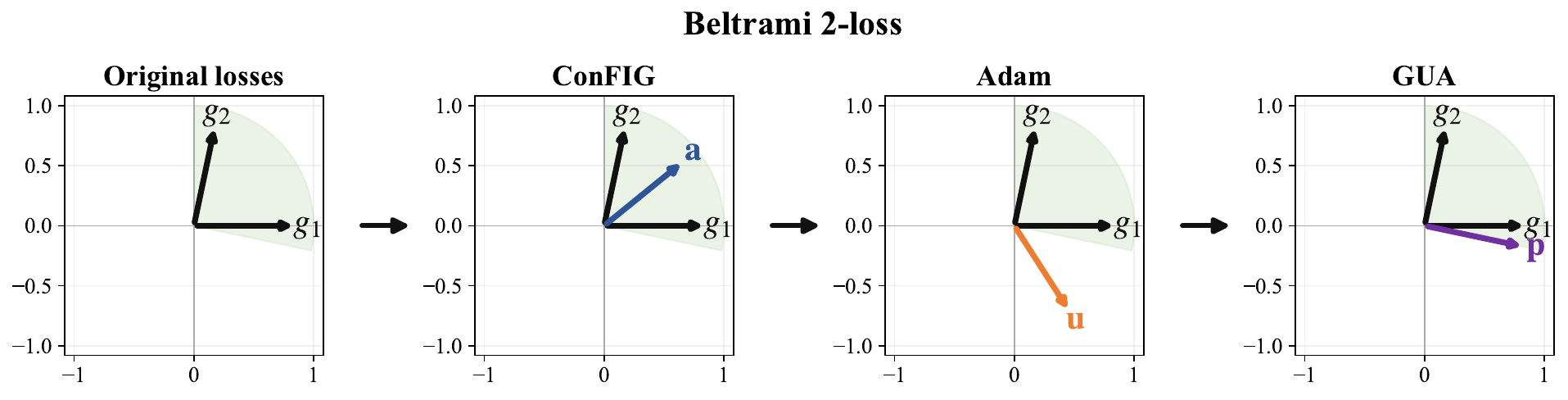}
\end{minipage}
\par\vspace{0.25em}
\begin{minipage}[t]{0.98\textwidth}
\centering
\includegraphics[width=\linewidth,height=0.135\textheight,keepaspectratio]{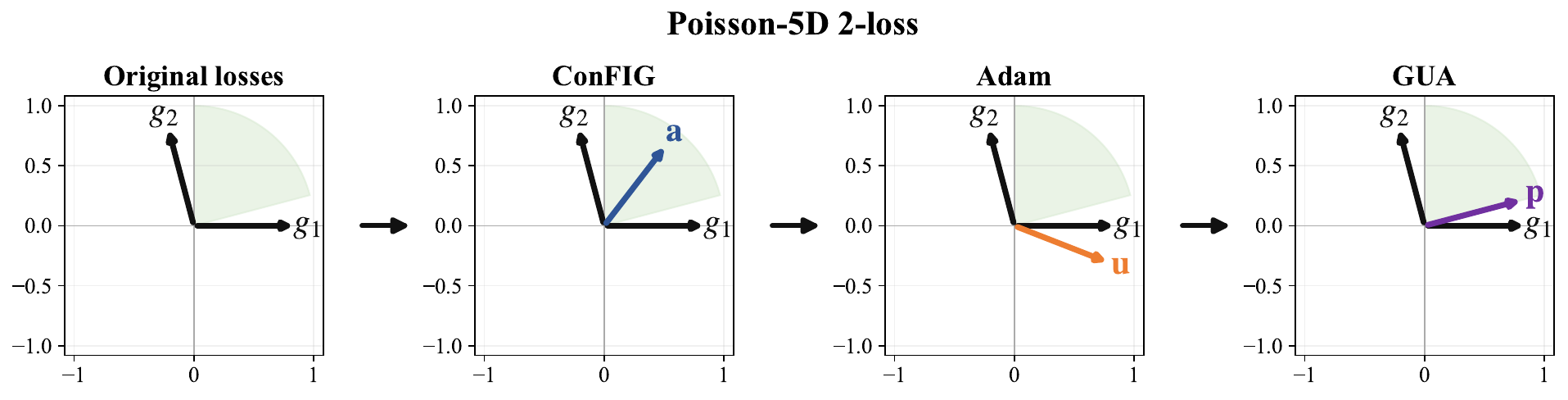}
\end{minipage}
\caption{Representative gradient directions for two-loss PDEs.}
\label{fig:gradient_direction_slices}
\end{figure*}

\clearpage

\section{Multi-task Learning Details}
\label{app:mtl_details}

We use CelebA~\cite{liu2015deep} as a controlled task-cardinality benchmark,
following the FAMO-based experimental setup adopted by
ConFIG~\cite{liu2025config}. For an $m$-task experiment, we use the first $m$
attributes in the official annotation order, so the task-count settings form
nested subsets. Images are resized to $64\times64$, converted to tensors, and
used without additional data augmentation. The model consists of a shared
convolutional backbone with BatchNorm, ReLU, max pooling, adaptive average
pooling, and two fully connected layers, followed by task-specific binary
classification heads.

Training minimizes binary cross-entropy independently for each attribute.
The main training hyperparameters are summarized in
Table~\ref{tab:celeba_training_settings}. To remain consistent with the gradient-surgery baselines, which operate on the shared parameters, GUA is applied to the same shared parameters after task-gradient aggregation and optimizer proposal. The
task-specific classification heads are updated by the native optimizer but are
not modified by GUA. The data pipeline, model, and task losses remain
unchanged. All
reported results are averaged over three independent seeds $\{0,1,2\}$.

\begin{table}[!htbp]
\centering
\caption{CelebA multi-task learning training configuration.}
\label{tab:celeba_training_settings}
\normalsize
\setlength{\tabcolsep}{5pt}
\renewcommand{\arraystretch}{1.08}
\begin{tabular}{@{}c>{\centering\arraybackslash}p{0.62\linewidth}@{}}
\toprule
\textbf{Item} & \textbf{Setting} \\
\midrule
Available tasks  & 40 binary facial attributes; evaluated task counts \(\{2,3,5,10,20,30,40\}\) \\
Input  & RGB images resized to $64\times64$; no augmentation \\
Loss  & Binary cross-entropy per task \\
Epochs  & 15 \\
Batch size  & 256 \\
Optimizer  & Adam \\
Learning rate  & $3\times10^{-4}$ \\
Evaluation frequency & Every epoch \\
\bottomrule
\end{tabular}
\end{table}

\end{document}